\documentclass{article}

\usepackage[letterpaper,margin=1in]{geometry}
\usepackage[parfill]{parskip}

\usepackage[utf8]{inputenc} 
\usepackage[T1]{fontenc}    
\usepackage[colorlinks]{hyperref}       
\usepackage{url}            
\usepackage{booktabs}       
\usepackage{amsfonts}       
\usepackage{amsmath,amsthm,amssymb,bbm}
\usepackage{mathtools}
\usepackage{nicefrac}       
\usepackage{microtype}      
\usepackage{natbib}
\usepackage{enumitem}
\usepackage{mathrsfs}
\usepackage{hyperref}
\hypersetup{colorlinks=true,linkcolor=blue,citecolor=blue}

\usepackage[english]{babel}
\usepackage{dsfont}

\usepackage{comment}
\usepackage{caption}
\usepackage{subcaption}

\usepackage{algorithm,algorithmic}

\usepackage{color}
\usepackage{appendix}

\definecolor{maroon}{RGB}{192,80,77}

\newcommand{\explain}[2]{\underset{\mathclap{\overset{\uparrow}{#2}}}{#1}}

\newtheorem{theorem}{Theorem}
\newtheorem{lemma}[theorem]{Lemma}
\newtheorem{proposition}[theorem]{Proposition}

\newtheorem{corollary}[theorem]{Corollary}
\newtheorem{definition}[theorem]{Definition}
\newtheorem{example}[theorem]{Example}

\newtheorem{remark}[theorem]{Remark}

\newcommand{\argmax}{\mathop{\mathrm{argmax}}}

\def\E{\mathbb{E}}

\def\subsample{\mathsf{subsample}}

\def\R{\mathbb{R}}

\def\cG{\mathcal{G}}

\def\cI{\mathcal{I}}

\def\cM{\mathcal{M}}
\def\cN{\mathcal{N}}

\def\cS{\mathcal{S}}

\def\cX{\mathcal{X}}

\title{Subsampled R\'enyi Differential Privacy and Analytical Moments Accountant}

\author{
	Yu-Xiang Wang\thanks{The research is partially completed while Yu-Xiang was a scientist in Amazon AI, Palo Alto.} \\
	UC Santa Barbara \\
	Santa Barbara, CA\\
	\texttt{yuxiangw@cs.ucsb.edu}\\
	\and
		Borja Balle\\
Amazon AI\\
Cambridge, UK \\
\texttt{pigem@amazon.co.uk} \\
\and
Shiva Kasiviswanathan\\
Amazon AI\\
Sunnyvale, CA  \\
\texttt{kasivisw@gmail.com} \\
}

\date{ }
\begin{document}
\maketitle

\begin{abstract}
We study the problem of subsampling in differential privacy (DP), a question that is the centerpiece behind many successful differentially private machine learning algorithms.  Specifically, we provide a tight upper bound on the R\'enyi Differential Privacy (RDP)~\citep{mironov2017renyi} parameters for algorithms that: (1) subsample the dataset, and then (2) applies a randomized mechanism $\cM$ to the subsample, in terms of the RDP parameters of $\cM$ and the subsampling probability parameter.
Our results generalize the moments accounting technique, developed by \cite{abadi2016deep} for the Gaussian mechanism,  to any subsampled RDP mechanism.
\end{abstract}
\section{Introduction}

Differential privacy (DP)  is a mathematical definition of privacy proposed by \citet{dwork2006calibrating}. Ever since its introduction, DP has been widely adopted and as of today, it has become the \emph{de facto} standard of privacy definition in the academic world with also wide adoption in the industry~\citep{erlingsson2014rappor,apple2017,uber2017}. DP provides provable protection against adversaries with arbitrary side information and computational power,  allows clear quantification of privacy losses, and satisfies graceful composition over multiple access to the same data. Over the past decade, a large body of work has been developed to design basic algorithms and tools for  achieving differential privacy, understanding the privacy-utility trade-offs in different data access setups, and on integrating differential privacy with machine learning and statistical inference.  We refer the reader to~\citep{dwork2013algorithmic} for a more comprehensive overview.

R\'enyi Differential Privacy (RDP, see Definition~\ref{def:RDP})~\citep{mironov2017renyi} is a recent refinement of differential privacy~\citep{dwork2006calibrating}.  It offers a unified view of the $\epsilon$-differential privacy (pure DP), $(\epsilon,\delta)$-differential privacy (approximate DP), and the related notion of {\em Concentrated Differential Privacy}~\citep{dwork2016concentrated,bun2016concentrated}. The RDP point of view on differential privacy is particularly useful when the dataset is accessed by a sequence of randomized mechanisms, as in this case a {\em moments accountant} technique can be used to effectively keep track of the usual $(\epsilon,\delta)$ DP parameters across the entire range $\{ (\epsilon(\delta),\delta) |  \forall \delta\in [0,1]\}$~\citep{abadi2016deep}.

A prime use case for the moments accountant technique is the {\em NoisySGD} algorithm~\citep{song2013stochastic,bassily2014private} for differentially private learning, which iteratively executes:
\begin{align} \label{eqn:noisySGD}
	\theta_{t+1}  \leftarrow  \theta_t  +  \eta_t \left( \frac{1}{|\cI|}\sum_{i\in \cI}\nabla f_i(\theta_t) +  Z_t\right)
\end{align}
where $\theta_t$ is the model parameter at $t$th step, $\eta_t$ is the learning rate, $f_i$ is the loss function of data point $i$, $\nabla$ is the standard gradient operator, $\cI$ is an index set of size $m$ that we uniformly randomly drawn from $\{1,...,n\}$, and $Z_t\sim \cN(0, \sigma^2 I)$.  Adding Gaussian noise (also known as the {\em Gaussian mechanism}) is a standard way of achieving $(\epsilon,\delta)$-differential privacy~\citep{dwork2006our,dwork2013algorithmic,balle2018improving}. Since in the NoisySGD case the randomized algorithm first chooses (subsamples) the mini-batch $\cI$ randomly before adding the Gaussian noise, the overall scheme could be viewed as a {\em subsampled Gaussian mechanism}. Therefore, with the right setting of $\sigma$, each iteration of NoisySGD can be thought of as a private release of a stochastic gradient. 

More generally, a subsampled randomized algorithm first takes a subsample of the dataset generated through some subsampling procedure\footnote{There are different subsampling methods, such as Poisson subsampling, sampling without replacement, sampling with replacement, etc.}, and then applies a known randomized mechanism $\cM$ on the subsampled data points. It is important to exploit the randomness in subsampling because if $\cM$ is $(\epsilon,\delta)$-DP, then (informally) a subsampled mechanism obeys $(O(\gamma \epsilon),\gamma\delta)$-DP for some $\gamma < 1$ related to the sampling procedure.  This is often referred to as the ``privacy amplification'' lemma\footnote{Informally, this lemma states that, if a private algorithm is run on a random subset of a larger dataset (and the identity of that subset remains hidden), then this new algorithm provides better privacy protection (reflected through improved privacy parameters) to the entire dataset as a whole than the original algorithm did.} --- a key property that enables NoisySGD and variants to achieve optimal rates in convex problems \citep{bassily2014private}, and to work competitively in Bayesian learning~\citep{wang2015privacy} and deep learning~\citep{abadi2016deep} settings. A side note is that privacy amplification is also the key underlying technical tool for characterizing the  learnability in statistical learning \citep{wang2016learning} and achieving tight sample complexity bounds for simple function classes \citep{beimel2013characterizing,bun2015differentially}.

While privacy amplification via subsampling is a very important tool for designing good private algorithms, computing the RDP parameters for a subsampled mechanism is a non-trivial task.  A natural question, with wide ranging implications for designing successful differentially private algorithms is the following: Can we obtain good bounds for privacy parameters of a subsampled mechanism in terms of privacy parameters of the original mechanism?
With the exception of the special case of the Gaussian mechanism under Poisson subsampling analyzed in~\citep{abadi2016deep}, there is no analytical formula available to generically convert the RDP parameters of a mechanism $\cM$ to the RDP parameters of the subsampled mechanism. 

In this paper, we tackle this central problem in private data analysis and provide the first general result in this area. 
Specifically, we analyze RDP amplification under a \emph{sampling without replacement} procedure: $\subsample$, which takes a data set of $n$ points and outputs a sample from the uniform distribution over all subsets of size $m\leq n$. 
%
Our contributions can be summarized as follows:
\begin{list}{{\bf (\roman{enumi})}}
	{\usecounter{enumi}
}
	\item We provide a tight bound (Theorem~\ref{thm:main}) on the RDP parameter ($\epsilon_{\cM\circ \subsample}(\alpha)$) for a subsampled mechanism ($\cM\circ\subsample$) in terms of the RDP parameter ($\epsilon_{\cM}(\alpha)$) of the original mechanism ($\cM$) itself and the subsampling ratio $\gamma := m/n$. Here, $\alpha$ is the order of the R\'enyi divergence in the RDP definition (see Definition~\ref{def:RDP} and the following discussion). This is the first general result in this area that can be applied to any RDP mechanism.
	For example, in addition to providing RDP parameter bounds for the subsampled Gaussian mechanism case, our result  enables analytic calculation of similar bounds for many more commonly used privacy mechanisms including subsampled Laplace mechanisms, subsampled randomized response mechanisms, subsampled ``posterior sampling'' algorithms under exponential family models~\citep{geumlek2017renyi}, etc. Even for the subsampled Gaussian mechanism our bounds are tighter than those provided by~\citet{abadi2016deep} (albeit the subsampling procedure and the dataset neighboring relation they use are slightly different from ours).
	
	\item Consider a mechanism $\cM$ with RDP parameter $\epsilon_{\cM}(\alpha)$. Interestingly, our bound on the RDP parameter of the subsampled mechanism indicates that as the order of RDP $\alpha$ increases, there is a phase transition point $\alpha^*$ satisfying $\gamma \alpha^* e^{\epsilon_\cM(\alpha^*)} \approx 1$.  For $\alpha<\alpha^*$,  the subsampled mechanism has an RDP parameter $\epsilon_{\cM\circ \subsample}(\alpha)=O(\alpha \gamma^2 (e^{\epsilon_{\cM}(2)}-1))$, while for $\alpha>\alpha^*$, the RDP parameter $\epsilon_{\cM\circ \subsample}(\alpha)$ either quickly converges to $\epsilon_{\cM}(\alpha)$ which does not depend on $\gamma$, or tapers off at $O(\gamma\epsilon_{\cM}(\infty))$ which happens when $e^{\epsilon_{\cM}(\infty)} - 1 \ll 1/\gamma$. The subsampled Gaussian mechanism falls into the first category, while the subsampled Laplace mechanism falls into the second. 
	
	
	\item Our analysis reveals  a new theoretical quantity of interest that has not been investigated before --- a \emph{ternary} version of the Pearson-Vajda divergence (formally defined in Appendix~\ref{app:subsampling}). A privacy definition defined through this divergence seems naturally coupled with understanding the effects of subsampling, just like how R\'enyi differential privacy (RDP)~\citep{mironov2017renyi} seems naturally coupled with understanding the effects of composition. 	
	
	
	\item From a computational efficiency perspective, we propose an efficient data structure to keep track of the R\'enyi differential privacy parameters in its symbolic form, and output the corresponding $(\epsilon,\delta)$-differential privacy as needed using efficient numerical methods. This avoids the need to specify a discrete list of moments ahead of time as required in the {\em moments accountant} method of~\citet{abadi2016deep} (see the discussion in Section~\ref{sec:ana_moment_acct}). Finally, our experiments confirm the improvements in privacy parameters that can be obtained by applying our bounds.
\end{list}

We end this introduction with a methodological remark. The main result of this paper is the bound in Theorem~\ref{thm:main}, which at first glance looks cumbersome. The remarks following the statement of the theorem in Section~\ref{sec:amp} discuss some of the asymptotic implications of this bound, as well as its meaning in several special cases. These provide intuitive explanations justifying the tightness of the bound. In practice, however, asymptotic bounds are of limited interest: concrete bounds with explicit, tight constants that can be efficiently computed are needed to provide the best possible privacy-utility trade-off in practical applications of differential privacy. Thus, our results should be interpreted under this point of view, which is summarized by the leitmotif \emph{``in differential privacy, constants matter''}. 

\section{Background and Related Work}\label{sec:background}

In this section, we review some background about differential privacy, some related privacy notions, and the technique of moments accountant. 

\noindent\textbf{Differential privacy and Privacy Loss Random Variable.} We start with the definition of $(\epsilon,\delta)$-differential privacy. We assume that $\cX$ is the domain that the datapoints are drawn from. We call two datasets $X$ and $X'$ {\em neighboring} (adjacent) if they differ in at most one data point, meaning that we can obtain $X'$ by \emph{replacing} one data point from $X$ by another arbitrary data point. We represent this as $d(X,X') \leq 1$.
\begin{definition}[Differential Privacy] \label{def:dp}
	A randomized algorithm $\cM : \cX^n \to \Theta$ is $(\epsilon,\delta)$-DP (differentially private) if for every pair of neighboring datasets $X,X'\in \cX^n$ (i.e., that differs only by one datapoint), and every possible (measurable) output set $E \subseteq \Theta$ the following inequality holds: $\Pr[\cM(X) \in E] \leq e^{\epsilon} \Pr[\cM(X') \in E] + \delta$.
\end{definition}
The definition ensures that it is information-theoretically impossible for an adversary to infer whether the input dataset is $X$ or $X'$ beyond a certain confidence, hence offering a degree of \emph{plausible deniability} to individuals in the dataset. Here, $\epsilon, \delta$ are what we call privacy loss parameters and the smaller they are, the stronger the privacy guarantee is. 
A helpful way to work with differential privacy is in terms of tail bounds on the {\em privacy loss random variable}. Let $\cM(X)$ and $\cM(X')$ be the probability distribution induced by $\cM$ on neighboring datasets $X$ and $X'$ respectively, the \emph{the privacy loss random variable} is defined as: $\log(\cM(X)(\theta) / \cM(X')(\theta))$ where $\theta \sim \cM(X)$. Up to constant factors, $(\epsilon,\delta)$-DP (Definition~\ref{def:dp}) is equivalent to requiring that the probability of the privacy loss random variable being greater than $\epsilon$ is at most $\delta$ for all neighboring datasets $X,X'$.\!\footnote{For meaningful guarantees, $\delta$ is typically taken to be ``cryptographically'' small.}  An important strength of differential privacy is the ability to reason about cumulative privacy loss under composition of multiple analyses on the same dataset.

Classical design of differentially private mechanisms takes these $\epsilon,\delta$ privacy parameters as inputs and then the algorithm carefully introduces some randomness to satisfy the privacy constraint (Definition~\ref{def:dp}), while simultaneously trying to achieve good utility (performance) bounds. However, this paradigm has shifted a bit recently as it has come to our realization that a more fine-grained analysis tailored for specific mechanisms could yield more favorable privacy-utility trade-offs and better privacy loss parameters under composition~\citep[See, e.g.,][]{dwork2016concentrated,abadi2016deep,balle2018improving}.

A common technique for achieving differential privacy while working with a real-valued function $f : \cX^n \rightarrow \R$ is via addition of noise calibrated to $f$'s sensitivity $S_f$, which is defined as the maximum of the absolute distance $|f(X) - f(X')|$ where $X,X'$ are adjacent inputs.\!\footnote{The restriction to a scalar-valued function is intended to simplify this presentation, but is not essential.} In this paradigm, the Gaussian mechanism is defined as: $\cG(X) := f(X) + \cN(0,S_f^2 \sigma^2)$. A single application of the Gaussian mechanism to a function $f$ with sensitivity $S_f$ satisfies $(\epsilon,\delta)$-differential privacy if\footnote{\citet{balle2018improving} show that a more complicated relation between $\epsilon$ and $\delta$ yields an if and only if statement.} $\delta \geq 0.8 \cdot \exp(-(\sigma \epsilon)^2/2)$ and $\epsilon \leq 1$~\citep[Theorem 3.22]{dwork2013algorithmic}.

\noindent\textbf{Stochastic Gradient Descent and Subsampling Lemma.}
A popular way of designing differentially private machine learning models is to use Stochastic Gradient Descent (SGD) with differentially private releases of (sometimes clipped) gradients evaluated on mini-batches of a dataset~\citep{song2013stochastic,bassily2014private,wang2015privacy,foulds2016theory,abadi2016deep}. Algorithmically, these methods are nearly the same and are all based on the NoisySGD idea presented in~\eqref{eqn:noisySGD}. They differ primarily in how they keep track of their privacy loss. \citet{song2013stochastic} uses a sequence of disjoint mini-batches to ensure each data point is used only once in every data pass. The results in \citep{bassily2014private,wang2016learning,foulds2016theory} make use of the privacy amplification lemma to take advantage of the randomness introduced by subsampling.  The first privacy amplification lemma appeared in~\citep{kasiviswanathan2011can,beimel2013characterizing}, with many subsequent improvements in different settings. For the case of $(\epsilon,\delta)$-DP, \citet{balle2018couplings} provide a unified account of privacy amplification techniques for different types of subsampling and dataset neighboring relations.
In this paper, we work in the subsampling without replacement setup, which satisfies the following privacy amplification lemma for $(\epsilon,\delta)$-DP. 
\begin{definition}[Subsample]
	Given a dataset $X$ of $n$ points, the procedure $\subsample$ selects a random sample from the uniform distribution over all subsets of $X$ of size $m$. The ratio $\gamma := m/n$ is defined as the sampling parameter of the $\subsample$ procedure.
\end{definition}
\begin{lemma}[\citep{ullman2017}\footnote{This result follows from Ullman's proof, though the notes state a weaker result. See also \citep{balle2018couplings}}]\label{lem:subsampling_approx}
	If $\cM$ is $(\epsilon,\delta)$-DP, then $\cM'$ that applies $\cM\circ\subsample$ obeys $(\epsilon',\delta')$-DP with $\epsilon' =\log\big(1 + \gamma (e^{\epsilon}-1)\big)$ and $\delta' = \gamma \delta$.
\end{lemma}
Roughly, the lemma says that subsampling with probability $\gamma <1$ amplifies an $(\epsilon,\delta)$-DP algorithm to an $(O(\gamma\epsilon),\gamma \delta)$-DP algorithm for a sufficiently small choice of $\epsilon$. The overall differentially private guarantees in~\citep{wang2015privacy,bassily2014private,foulds2016theory} were obtained by keeping track of the privacy loss over each iterative update of the model parameters using the {\em strong composition theorem} in differential privacy~\citep{dwork2010boosting}, which gives roughly $(\tilde{O}(\sqrt{k}\epsilon), \tilde{O}(k\delta))$-DP\footnote{The $\tilde{O}(\cdot)$ notation hides various logarithmic factors.}  for $k$ iterations of an arbitrary $(\epsilon,\delta)$-DP algorithm (see Appendix~\ref{app:dp} for a discussion about various composition results in differential privacy).

The work of \citet{abadi2016deep} was the first to take advantage of the fact that $\cM$ is a subsampled Gaussian mechanism and used a mechanism-specific way of doing the strong composition.
Their technique, referred to as {\em moments accountant}, is described below.

\noindent\textbf{Cumulant Generating Functions, Moments Accountant, and R\'enyi Differential Privacy.}
The moments accountant technique of~\citet{abadi2016deep} centers around the cumulant generating function (CGF, or the log of the moment generating function) of the privacy loss random variable:
\begin{align}
	K_\cM(X,X',\lambda) &:= \log \E_{\theta\sim\cM(X)}\Big[e^{\lambda \log \frac{\cM(X)(\theta)}{\cM(X')(\theta)} }\Big] = \log \E_{\theta\sim\cM(X)}\left[  \left(\frac{\cM(X)(\theta)}{\cM(X')(\theta)} \right)^{\lambda} \right].\label{eq:data_dependent_cgf}
\end{align}
After a change of measure, this is equivalent to:
$$ K_\cM(X,X',\lambda) :=  \log \E_{\theta\sim\cM(X')}\left[  \left(\frac{\cM(X)(\theta)}{\cM(X')(\theta)} \right)^{\lambda+1} \right].$$
Two random variables have identical CGFs then they are identically distributed (almost everywhere). In other words, this function characterizes the entire distribution of the privacy loss random variable.

Before explaining the details behind the moments accountant technique, we introduce the notion of R\'enyi differential privacy (RDP)~\citep{mironov2017renyi} as a generalization of differential privacy that uses the $\alpha$-R\'enyi divergences between $\cM(X)$ and $\cM(X')$. 
\begin{definition}[R\'enyi Differential Privacy] \label{def:RDP}
	We say that a mechanism $\cM$ is $(\alpha, \epsilon)$-RDP with order $\alpha\in (1,\infty)$ if for all neighboring datasets $X,X'$ 
	\begin{align*}
&D_{\alpha}(\cM(X)\|\cM(X')  )
:= \frac{1}{\alpha-1}\log \E_{\theta\sim \cM(X')}\left[ \left(\frac{\cM(X)(\theta)}{\cM(X')(\theta)}\right)^\alpha \right] \leq \epsilon.
\end{align*}
\end{definition}
As $\alpha\rightarrow \infty$ RDP reduces to $(\epsilon,0)$-DP (pure DP), i.e., a randomized mechanism $\cM$ is $(\epsilon,0)$-DP if and only if for any two adjacent inputs $X$ and $X'$ it satisfies $D_{\infty}( \cM(X)\| \cM(X')  ) \leq \epsilon$. For $\alpha\rightarrow 1$, the RDP notion reduces to Kullback-Leibler based privacy notion, which is equivalent to a bound on the expectation of the privacy loss random variable. For a detailed exposition of the guarantee and properties of R\'enyi differential privacy that mirror those of differential privacy, see Section III of \citet{mironov2017renyi}. Here, we highlight two key properties that are relevant for this paper.

\begin{lemma}[Adaptive Composition of RDP, Proposition~1 of \citep{mironov2017renyi}]\label{lem:composition_RDP}
	If $\cM_1$ that takes dataset as input obeys $(\alpha,\epsilon_1)$-RDP, and $\cM_2$ that takes the dataset and the output of $\cM_1$ as input obeys $(\alpha,\epsilon_2)$-RDP,  then their composition obeys 
	$(\alpha,\epsilon_1+\epsilon_2)$-RDP.
\end{lemma}
\begin{lemma}[RDP to DP conversion, Proposition~3 of \citep{mironov2017renyi}]\label{lem:RDP2DP}
	If $\cM$ obeys $(\alpha,\epsilon)$-RDP, then $\cM$ obeys $(\epsilon + \log(1/\delta)/(\alpha-1),\delta)$-DP for all $0<\delta<1$.
\end{lemma}

\noindent\textbf{RDP Functional View.} While RDP for each fixed $\alpha$  can be used as a standalone privacy measure, we emphasize its \emph{functional view} in which $\epsilon$ is a function of $\alpha$ for $1\leq \alpha \leq \infty$, and this function is completely determined by $\cM$. This is denoted by $\epsilon_{\cM}(\alpha)$, and with this notation, mechanism $\cM$ satisfies $(\alpha,\epsilon_{\cM}(\alpha))$-RDP in Definition~\ref{def:RDP}. In other words,
$$ \sup_{X,X' : d(X,X') \leq 1}\, D_{\alpha}(\cM(X)\|\cM(X')) \leq \epsilon_{\cM}(\alpha).$$
Here $\epsilon_{\cM}(\alpha)$ is referred to as the RDP parameter. We drop the subscript from $\epsilon_{\cM}$ when $\cM$ is clear from the context.  We use $\epsilon_{\cM}(\infty)$  (or $\epsilon(\infty)$) to denote the case where $\alpha = \infty$, which indicates that the mechanism $\cM$ is $(\epsilon,0)$-DP (pure DP) with $\epsilon = \epsilon(\infty)$.

Our goal is, given a mechanism $\cM$ that satisfies $(\alpha,\epsilon(\alpha))$-RDP, to investigate the RDP parameter of the subsampled mechanism $\cM\circ\subsample$, i.e., to get a bound on $\epsilon_{\cM\circ\subsample}(\alpha)$ such that the mechanism $\cM\circ\subsample$ satisfies $(\alpha, \epsilon_{\cM\circ\subsample}(\alpha))$-RDP.

Note that $\epsilon_{\cM}(\alpha)$ is equivalent to a data-independent upper bound of the CGF (as defined in \eqref{eq:data_dependent_cgf}),
$$K_\cM(\lambda):= \sup_{X,X' : d(X,X') \leq 1}\,K_\cM(X,X',\lambda),$$
up to a scaling transformation (with $\alpha = \lambda +1$) as noted by the following remark.
\begin{remark}[RDP $\Leftrightarrow$ CGF]\label{rmk:cgf2renyi}
	A randomized mechanism $\cM$ obeys $(\lambda +1,K_\cM(\lambda) / \lambda)$-RDP for all $\lambda$. 
\end{remark}

The idea of moments accountant~\citep{abadi2016deep} is to essentially keep track of the evaluations of CGF at a list of fixed locations through Lemma~\ref{lem:composition_RDP} and then Lemma~\ref{lem:RDP2DP} allows one to find the smallest $\epsilon$ given a desired $\delta$ or vice versa using:
\begin{align}
	\delta \Rightarrow\epsilon:   \quad\quad \epsilon(\delta) &= \min_{\lambda} \frac{\log(1/\delta)+K_{\cM}(\lambda)}{\lambda}  \label{eq:eps_from_delta},\\
	\epsilon \Rightarrow\delta:   \quad\quad \delta(\epsilon) &= \min_{\lambda}e^{K_\cM(\lambda) - \lambda\epsilon} \label{eq:delta_from_eps}.
\end{align}
Using the convexity of CGF $K_{\cM}(\lambda)$ and monotonicity of $K_{\cM}(\lambda)/\lambda$ in $\lambda$ \citep[Corollary 2, Theorem 3]{van2014renyi}, we observe that the 
optimization problem in~\eqref{eq:delta_from_eps} is log-convex and the optimization problem~\eqref{eq:eps_from_delta} is unimodal/quasi-convex. Therefore, the optimization problem in~\eqref{eq:eps_from_delta} (similarly, in~\eqref{eq:delta_from_eps})  can be solved to an arbitrary accuracy $\tau$ in time $\log(\lambda^\ast/\tau)$ using the bisection method, where $\lambda^\ast$ is the optimal value for $\lambda$ from~\eqref{eq:eps_from_delta} (similarly,~\eqref{eq:delta_from_eps}). The same result holds even if all we have is (possibly noisy) blackbox access to $K_{\cM}(\cdot)$ or its derivative (see more details in Appendix~\ref{app:ana_moment_accountant}).

For other useful properties of the CGF and an elementary proof of its convexity and how it implies the monotonicity of the R\'enyi divergence, see Appendix~\ref{app:CGF_properties}.

\noindent\textbf{Other Related Work.}
A closely related notion to RDP is that of \emph{zero-concentrated differential privacy} (zCDP) introduced in~\citep{bun2016concentrated} (see also~\citep{dwork2016concentrated}). zCDP is related to CGF of the privacy loss random variable as we note here.
\begin{remark}[Relation between CGF and Zero-concentrated Differential Privacy]
	If randomized mechanism $\cM$ obeys $(\xi,\rho)$-zCDP for some parameters $\xi,\rho$, then the CGF $K_\cM(\lambda) \leq  \lambda \xi + \lambda(\lambda +1) \rho$. On the other hand, if $\cM$'s privacy loss r.v.\ has CGF
	$K_\cM(\lambda)$, then $\cM$ is also $(\xi,\rho)$-zCDP for all $(\xi,\rho)$ such that the quadratic function $\lambda \xi + \lambda(\lambda +1) \rho \geq K_\cM(\lambda)$.
\end{remark}
In general, the RDP view of privacy is broader than the CDP view as it captures finer information. For CDP, subsampling does not improve the privacy parameters~\citep{bun2018tcdp}. A truncated variant of the zCDP has been very recently proposed by~\citet{bun2018tcdp} and they studied the effect of subsampling in tCDP. While this independent work attempts to solve a problem closely related to ours, they are not directly comparable in that they deal with the amplification properties of tCDP while we deal with that of R\'enyi DP (and therefore CDP without truncation). A simple consequence of this difference is that the popular subsampled Gaussian mechanism explained above, that is covered by our analysis, is not directly covered by the amplification properties of tCDP. 


\section{Our Results}
In this section, we present first our main result, an amplification theorem for R\'enyi Differential Privacy via subsampling. We first provide the upper bound, and then discuss the optimality of this bound. Based on these bounds, in Section~\ref{sec:ana_moment_acct}, we discuss an idea for implementing a data structure that can efficiently track privacy parameters under composition.

\subsection{``Privacy Amplification'' for RDP} \label{sec:amp}

We start with our main theorem that bounds $\epsilon_{\cM\circ\subsample}(\alpha)$ for the mechanism $\cM\circ\subsample$ in terms of $\epsilon_{\cM}(\alpha)$ of the mechanism $\cM$ and sampling parameter $\gamma$ used in the $\subsample$ procedure. Missing details from this Section are collected in Appendix~\ref{app:subsampling}.

\begin{theorem}[RDP for Subsampled Mechanisms]\label{thm:main}
	Given a dataset of $n$ points drawn from a domain $\cX$ and a (randomized) mechanism $\cM$ that takes an input from $\cX^{m}$ for $m \leq n$, let the randomized algorithm $\cM\circ \subsample$ be defined as: (1) $\subsample$: subsample without replacement $m$ datapoints of the dataset (sampling parameter $\gamma = m/n$), and (2) apply $\cM$: a randomized algorithm taking the subsampled dataset as the input.
	For all integers $\alpha \geq 2$, if $\cM$ obeys $(\alpha,\epsilon(\alpha))$-RDP,  then this new randomized algorithm $\cM\circ \subsample$ obeys $(\alpha,\epsilon'(\alpha))$-RDP where, 
	\begin{align*}
		\epsilon'(\alpha)  \leq \frac{1}{\alpha-1}\log\bigg( 1 +  \gamma^2{\alpha \choose 2} \min\Big\{ 4(e^{\epsilon(2)}-1), e^{\epsilon(2)} \min\{2, (e^{\epsilon(\infty)}-1)^{2} \} \Big\}& \\
		+   \sum_{j=3}^{\alpha} \gamma^j {\alpha \choose j} e^{(j-1)\epsilon(j)} \min\{2, (e^{\epsilon(\infty)}-1)^{j}\}& \bigg).
	\end{align*}
\end{theorem}
The bound in the above theorem might appear complicated, and this is partly because of our efforts to get a precise non-asymptotic bound (and not just a $O(\cdot)$ bound) that can be implemented in a real system. Some additional practical considerations related to evaluating the bound in this theorem such as computational resources needed, numerical stability issues, etc., are discussed in Appendix~\ref{app:ana_moment_accountant}. 
The phase transition behavior of this bound, noted in the introduction, is probably most easily observed through Figure~\ref{fig:composedRDP} (Section~\ref{sec:exp}), where we empirically illustrates the behavior of this bound for the commonly used subsampled mechanisms. Now before discussing the proof idea, we mention few remarks about this result. 

\noindent\textbf{Generality.} Our results cover any R\'enyi differentially private mechanism, including those based on any exponential family distribution \citep[see][and our exposition in Appendix~\ref{app:expfamily_renyi}]{geumlek2017renyi}. As mentioned earlier, previously such a bound (even asymptotically) was only known for the special case of the subsampled Gaussian mechanism~\citep{abadi2016deep}. 	

\noindent\textbf{Pure DP.} In particular, Theorem~\ref{thm:main} also covers pure-DP mechanisms (such as Laplace and randomized response mechanisms) with a bounded $\epsilon(\infty)$. In this case, we can upper bound everything within the logarithm of Theorem~\ref{thm:main} with a binomial expansion:
$$
1  +  \sum_{j=1}^\alpha \gamma^j {\alpha \choose j}  e^{j\epsilon(\alpha)}  (e^{\epsilon(\infty)}-1)^j    =  \big(1+ \gamma e^{\epsilon(\alpha)} (e^{\epsilon(\infty)}-1)\big)^\alpha,
$$
which results in a bound of the form
$$
\epsilon'(\alpha) \leq \frac{\alpha}{\alpha-1} \log  \big(1+ \gamma e^{\epsilon(\alpha)} (e^{\epsilon(\infty)}-1)\big).
$$
As $\alpha\rightarrow \infty$ the expression converges to 
$
\log\left( 1 + \gamma e^{\epsilon(\infty)} (e^{\epsilon(\infty)}-1) \right)
$
which gives quantitatively the same result as the privacy amplification result in Lemma~\ref{lem:subsampling_approx} for the pure $(\epsilon,0)-$DP, modulo an extra $e^{\epsilon(\infty)}$ factor which becomes negligible as $\epsilon(\infty)$ gets smaller.

\noindent\textbf{Bound under Additional Assumptions.} 
The bound in Theorem~\ref{thm:main} could be strengthened under additional assumptions on the RDP guarantee. We defer a detailed discussion on this topic to Appendix~\ref{app:tight} (see Theorem~\ref{thm:tight}), but note that a consequence of this is that one can replace $e^{(j-1)\epsilon(j)}\min\{2,(e^{\epsilon(\infty)}-1)^{j}\}$ in the above bound with an exact evaluation given by the forward finite difference operator of some appropriately defined functional. Also we note that these additional assumptions hold for the Gaussian mechanism.

In particular, with subsampled Gaussian mechanism for functions with sensitivity $1$ (i.e., $\epsilon(\alpha) = \alpha/(2\sigma^2)$) the dominant part of the upper bound on $\epsilon'(\alpha)$  arises from the term $\min\{ 4(e^{\epsilon(2)}-1), e^{\epsilon(2)} \min\{2, (e^{\epsilon(\infty)}-1)^{2} \} \}$. Firstly, since the Gaussian mechanism does not have a  bounded $\epsilon(\infty)$ term, this term can be simplified as  $\min\{ 4(e^{\epsilon(2)}-1),2e^{\epsilon(2)}\}$. Let us consider the regimes: (a) $\sigma^2$ ``large'', (b) $\sigma^2$ ``small''. When $\sigma^2$ is large, $4(e^{\epsilon(2)}-1) = 4(e^{1/\sigma^2}-1) \leq 8/\sigma^2$ becomes the tight term in $\min\{4(e^{\epsilon(2)}-1), 2e^{\epsilon(2)} \}$. In this case, for small $\alpha$ and $\gamma$, the overall $\epsilon'(\alpha)$ bound simplifies to $O(\gamma^2\alpha/\sigma^2)$ (matching the asymptotic bound given in Appendix~\ref{app:gaussasym}). When $\sigma^2$ is small, then the $2e^{\epsilon(2)} = 2e^{1/\sigma^2}$ becomes the tight term in $\min\{4(e^{\epsilon(2)}-1), 2e^{\epsilon(2)} \}$. This (small  $\sigma^2$) is a regime that the results of~\citet{abadi2016deep} do not cover.

\noindent\textbf{Integer to Real-valued $\alpha$.}
The above calculations rely on a binomial expansion and thus only work for integer $\alpha$'s. To apply it to any real-valued, we can use the relation between RDF and CGF mentioned in Remark~\ref{rmk:cgf2renyi}, and the fact that CGF is a convex function (see Lemma~\ref{lem:properties} in Appendix~\ref{app:CGF_properties}). The convexity of $K_\cM(\cdot)$ implies that a piecewise linear interpolation yields a valid upper bound for all $\alpha\in (1,\infty)$.
\begin{corollary} \label{cor:ext}
	Let $\lfloor \cdot \rfloor$ and $\lceil \cdot \rceil$ denotes the floor and ceiling operators. Then, ${K_{\cM}( \lambda)} \leq (1-\lambda + \lfloor \lambda \rfloor) K_{\cM}(\lfloor \lambda \rfloor)  + (\lambda -  \lfloor \lambda \rfloor) K_{\cM}(\lceil \lambda \rceil)$.
\end{corollary}
The bound on $K_{\cM}( \lambda)$ can be translated into a RDP parameter bound as noted in Remark~\ref{rmk:cgf2renyi}.

\noindent\textbf{Proof Idea} The proof of this theorem is roughly split into three parts (see Appendix~\ref{app:mainproof}).  In the first part, we define a new family of privacy definitions called \emph{ternary-$|\chi|^\alpha$-differential privacy} (based on ternary version of Pearson-Vajda divergence) and show that it handles subsampling naturally (Proposition~\ref{prop:subsample_ternary}, Appendix~\ref{app:mainproof}).  In the second part, we bound the R\'enyi DP using the ternary-$|\chi|^\alpha$-differential privacy and apply the subsampling lemma from the first part. In the third part, we propose a number of ways of converting the expression stated as ternary-$|\chi|^\alpha$-differential privacy back to that of RDP (Lemmas~\ref{lem:ternary2binary},~\ref{lem:triangular_of_diff_pow},~\ref{lem:pure_dp_bound}, Appendix~\ref{app:mainproof}). Each of these conversion strategies yield different coefficients in the sum inside the logarithm defining $\alpha'(\epsilon)$; our bound accounts for all these strategies at once by taking the minimum of these coefficients.

\subsection{A lower bound of the RDP for subsampled mechanisms}
We now discuss whether our bound in Theorem~\ref{thm:main} can be improved. First, we provide a short answer: it cannot be improved in general. 
\begin{proposition}\label{prop:lowerbound}
	Let $\cM$ be a randomized algorithm that takes a dataset in $\cX^{\gamma n}$ as an input. If $\cM$ obeys $(\alpha,\epsilon(\alpha))$-RDP for a function $\epsilon: \R_+ \rightarrow \R_+$ and that there exists $x,x'\in\cX$ such that $\epsilon(\alpha) = D_\alpha\big( {\cM([x,x,...,x,x'])}\|  {\cM([x,x,...,x,x])}\big)$  for all integer $\alpha\geq 1$ (e.g., this condition is true for all output perturbation mechanisms for counting queries), then the RDP function $\epsilon'$ for $\cM\circ\subsample$ obeys the following lower bound for all integers $\alpha\geq 1$:
	\begin{align*}
		\epsilon'(\alpha) \geq \frac{\alpha}{\alpha-1}\log(1-\gamma)
		+ \frac{1}{\alpha-1} \log\Big( 1 + \alpha\frac{\gamma}{1-\gamma} + \sum_{j=2}^\alpha {\alpha \choose j}  \big(\frac{\gamma}{1-\gamma}\big)^j  e^{(j-1)\epsilon(j)}\Big).
	\end{align*}
\end{proposition}
\begin{proof}
	Consider two datasets $X,X' \in \cX^n$ where $X'$ contains $n$ data points that are identically $x$ and $X$ is different from $X'$ only in its last data point. By construction,  $\subsample(X')\equiv [x,x,...,x]$, $\Pr[\subsample(X) =  [x,x,...,x]] = 1-\gamma$ and $\Pr[\subsample(X) =  [x,x,...,x,x'] = \gamma$. In other words,
	$\cM\circ\subsample(X') = \cM([x,x,...,x]) := p $ and $\cM\circ\subsample(X) = (1-\gamma) p +   \gamma \cM([x,x,...,x,x']) := (1-\gamma)p+ \gamma q.$ 
	It follows that
	\begin{align*}
	\E_q\left[ \left(\frac{(1-\gamma)q + \gamma p}{q}\right)^\alpha \right] =& 	\E_q\left[ \left(1-\gamma + \gamma \frac{p}{q}\right)^\alpha \right] =  (1-\gamma)^\alpha\E_q\left[ \left(1 + \frac{\gamma}{1-\gamma} \frac{p}{q}\right)^\alpha \right] \\
	=& (1-\gamma)^{\alpha}\left(1 +  \alpha\frac{\gamma}{1-\gamma} + \sum_{j=2}^\alpha {\alpha \choose j}  \left (\frac{\gamma}{1-\gamma} \right )^j \E_q\left[\Big(\frac{p}{q}\Big)^j\right]\right).
	\end{align*}
	When we take  $x,x'$ to be the one in the assumption that attains the RDP $\epsilon(\cdot)$ upper bound, then we can replace $\E_q\left[(p/q)^j\right]$ in the above bound with $e^{(j-1)\epsilon(j)}$ as claimed.
\end{proof}

Let us compare the above lower bound to our upper bound in Theorem~\ref{thm:main} in two regimes. When $\alpha\gamma e^{\epsilon(\alpha)} \ll 1$, such that $\alpha^2\gamma^2e^{\epsilon(2)}<1$ is the dominating factor in the summation, we can use the bounds $x/(1+x)\leq \log(1+x) \leq x$ to get that both the upper and lower bound are $\Theta(\alpha \gamma^2 e^{\epsilon(2)})$. In other words, they match up to a constant multiplicative factor. For other parameter configurations, note that $\gamma/(1-\gamma)> \gamma$, our bound in Theorem~\ref{thm:main} (with the $2e^{(j-1)\epsilon(j)}$) is tight up to an additive factor $\frac{\alpha}{\alpha-1}\log((1-\gamma)^{-1}) +  \frac{\log(2)}{\alpha-1}$ which goes to $0$ as $\gamma\rightarrow 0$ and $\alpha \rightarrow \infty$. We provide explicit comparisons of the upper and lower bounds in the numerical experiments presented in Section~\ref{sec:exp}.

The longer answer to this question of optimality is more intricate. The RDP bound can be substantially improved when we consider more fine-grained per-instance RDP in the same flavor as the per-instance $(\epsilon,\delta)$-DP \citep{wang2017per}. The only difference from the standard RDP is that now $\epsilon$ is parameterized by a pair of fixed adjacent datasets.  This point is in illustrated in Appendix~\ref{app:gaussasym}, where we discuss an asymptotic approximation of the R\'enyi divergence for the subsampled Gaussian mechanism.

\subsection{Analytical Moments Accountant}\label{sec:ana_moment_acct}

Our theoretical results above allow us to build an analytical moments accountant for composing differentially private mechanisms. This is a data structure that tracks the CGF function $K_{\cM}(\cdot)$ of a (potentially adaptive) sequence of mechanisms $\cM$ in symbolic form (or as an evaluation oracle). It supports subsampling before applying $\cM$ and the $K_{\cM}(\cdot)$ will be adjusted accordingly using the RDP amplification bound in Theorem~\ref{thm:main}. The data structure allows data analysts to query the smallest $\epsilon$ from a given $\delta$ (or vice versa) for $(\epsilon,\delta)$-DP using~\eqref{eq:eps_from_delta} (or \eqref{eq:delta_from_eps}).


Practically, our analytical moments accountant is better than the moment accountants proposed by \citet{abadi2016deep} in several noteworthy ways: (1) our approach allows one to keep track the CGF's of all $\lambda\geq 1$ in symbolic form without paying infinite memory, whereas moments account~\citep{abadi2016deep} requires a predefined list of $\lambda$'s and pays a memory proportional to the size of the list; (2) our approach completely avoids numerical integration used by moments account; and finally (3) our approach supports subsampling for generic RDP mechanisms while the moments accountant was built for supporting only Gaussian mechanisms. All of this translates into an efficient and accurate way for tracking $\epsilon$'s and $\delta$'s when composing differentially private mechanisms.

We design the data structure to be numerically stable, and efficient in both space and time. In particular, it tracks CGFs with $O(1)$ time to compose a new mechanism and uses space only linear in the number of \emph{unique} mechanisms applied (rather than the number of total mechanisms applied). Using the convexity of CGFs and the monotonicity of RDP, we are able to provide $\delta \Rightarrow \epsilon$ conversion to $(\epsilon,\delta)$-DP to within accuracy $\tau$ in oracle complexity $O(\log(\lambda^\ast/\tau))$, where $\lambda^\ast$ is the optimal value for $\lambda$. Similarly, for $\epsilon\Rightarrow \delta$ queries.

Note that for subsampled mechanisms the direct evaluation $\epsilon_{\cM\circ\subsample}(\alpha)$ of the upper bounds in Theorem~\ref{thm:main} is already polynomial in $\alpha$. To make the data structure truly scalable, we devise a number of ways to approximate the bounds that takes only $O(\log(\alpha))$ evaluations of $\epsilon_{\cM}(\cdot)$.  More details about our analytical moments accountant and substantiations to the above claims are provided in Appendix~\ref{app:ana_moment_accountant}.

\section{Experiments and Discussion}\label{sec:exp}
In this section, we present numerical experiments to demonstrate our upper and lower bounds of RDP for subsampled mechanisms and the usage of analytical moments accountant. In particular, we consider three popular randomized privacy mechanisms: (1) Gaussian mechanism (2) Laplace mechanism, and (3) randomized response mechanism, and investigate the amplification effect of subsampling with these mechanisms on RDP.
The RDP of these three mechanisms are known in analytical forms \citep[See,][Table II]{mironov2017renyi} : 
\begin{align*}
&\epsilon_{\text{Gaussian}(\alpha)} =  \frac{\alpha}{2\sigma^2},\\
&\epsilon_{\text{Laplace}(\alpha)} = \frac{1}{\alpha-1}\log\left ( \left (\frac{\alpha}{2\alpha-1} \right ) e^{(\alpha-1)/\lambda} + \left ( \frac{\alpha-1}{2\alpha-1} \right )e^{-\alpha/\lambda} \right ) \text{ for }\alpha>1,\\
&\epsilon_{\text{RandResp}(\alpha)} = \frac{1}{\alpha-1}\log\left ( p^\alpha(1-p)^{1-\alpha} + (1-p)^{\alpha}p^{1-\alpha} \right ) \text{ for }\alpha>1.
\end{align*} 

\begin{figure}[t]
	\centering
	\begin{subfigure}[t]{0.32\textwidth}
		\includegraphics[width=\textwidth]{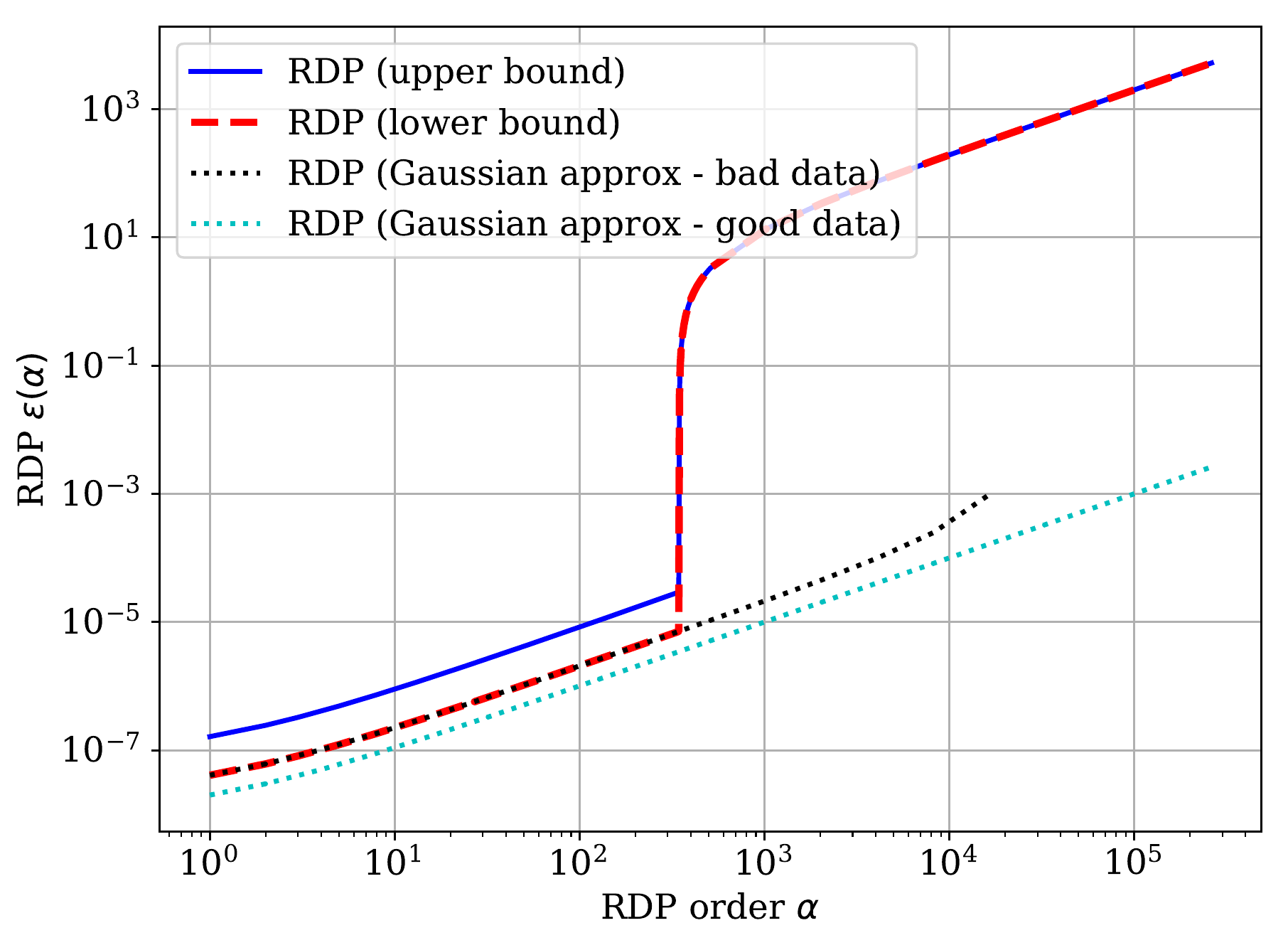}
		\caption{\scriptsize Subsampled Gaussian with $\sigma=5$.}
		\label{0a}
	\end{subfigure}
	\begin{subfigure}[t]{0.32\textwidth}
		\includegraphics[width=\textwidth]{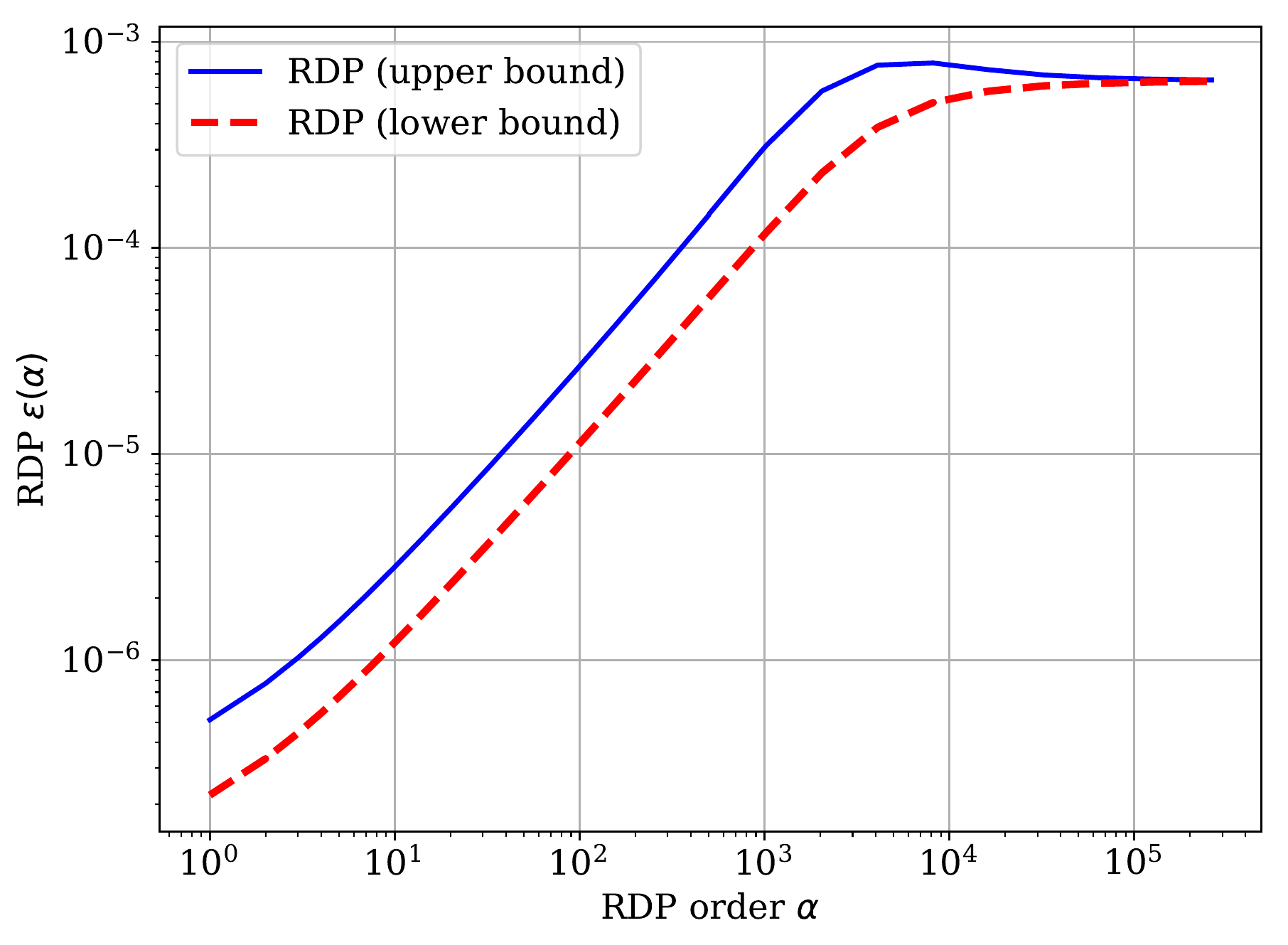}
		\caption{\scriptsize Subsampled Laplace with $b = 2$.}
		\label{0b}
	\end{subfigure}
	\begin{subfigure}[t]{0.32\textwidth}
		\includegraphics[width=\textwidth]{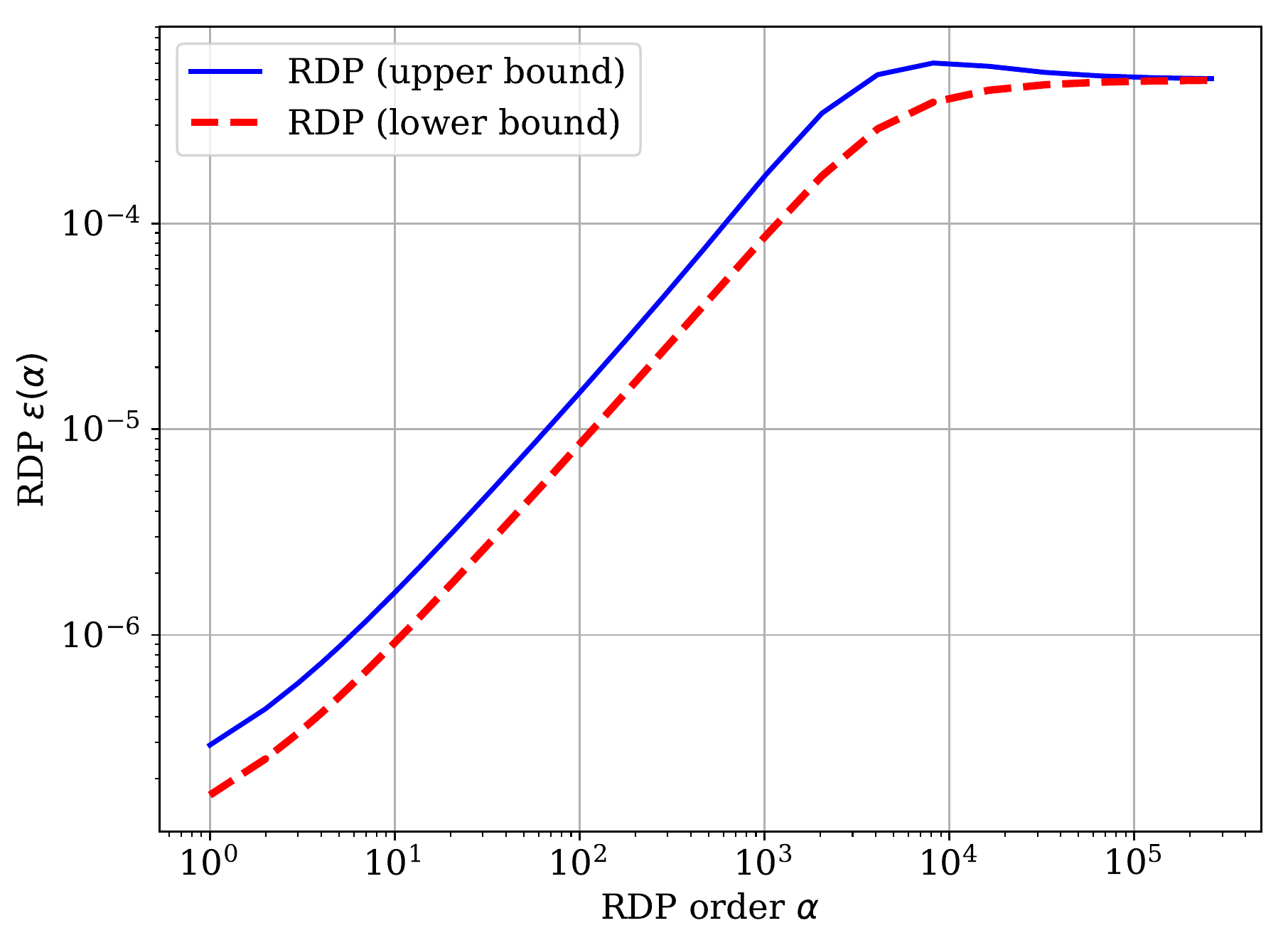}
		\caption{\scriptsize Subsampled Rand.\ Resp.\ with $p=0.6$.}
		\label{0c}
	\end{subfigure}\\
	
	\begin{subfigure}[t]{0.32\textwidth}
		\includegraphics[width=\textwidth]{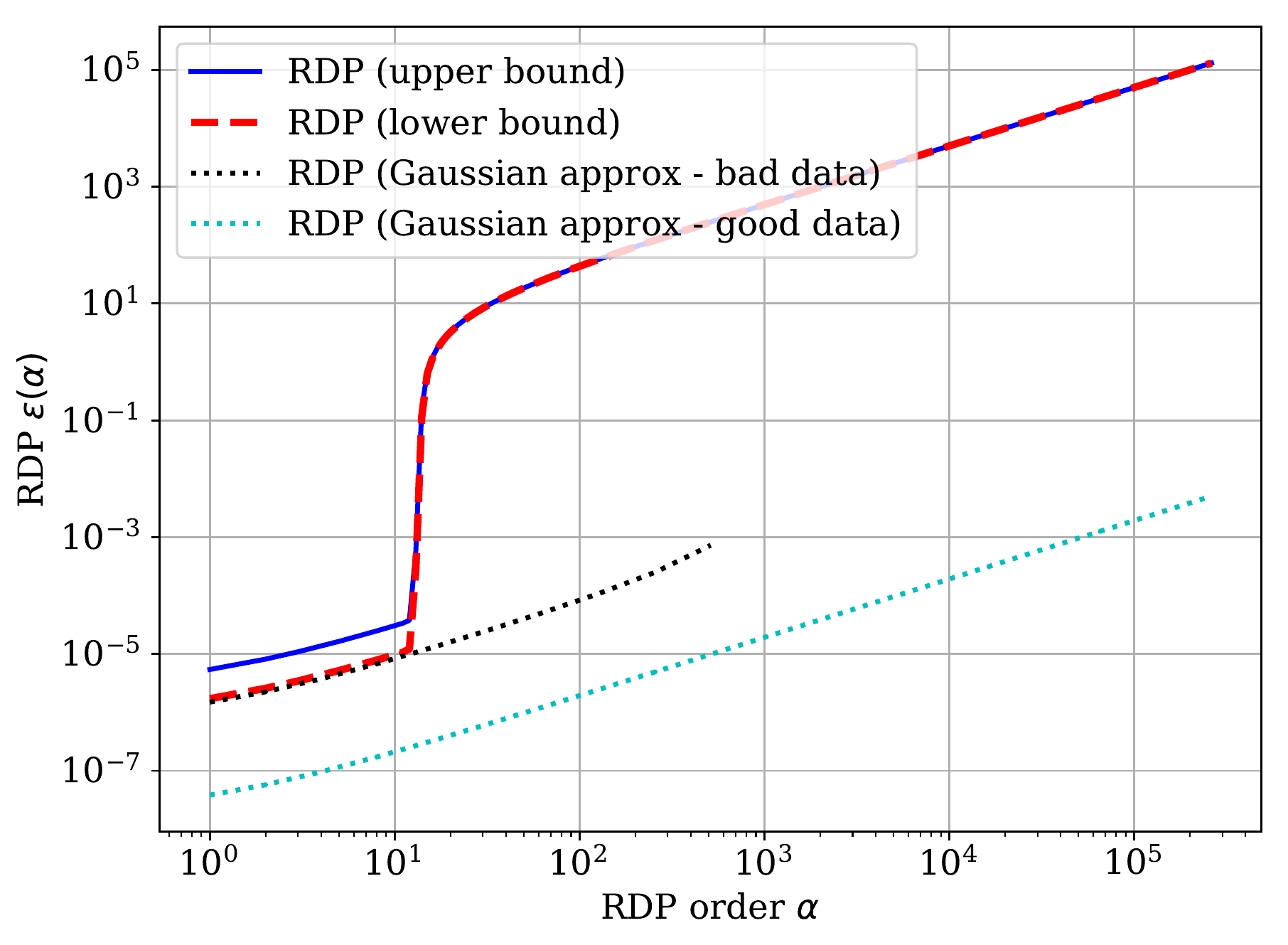}
		\caption{\scriptsize Subsampled Gaussian with $\sigma=0.5$.}
		\label{0d}
	\end{subfigure}
	\begin{subfigure}[t]{0.32\textwidth}
		\includegraphics[width=\textwidth]{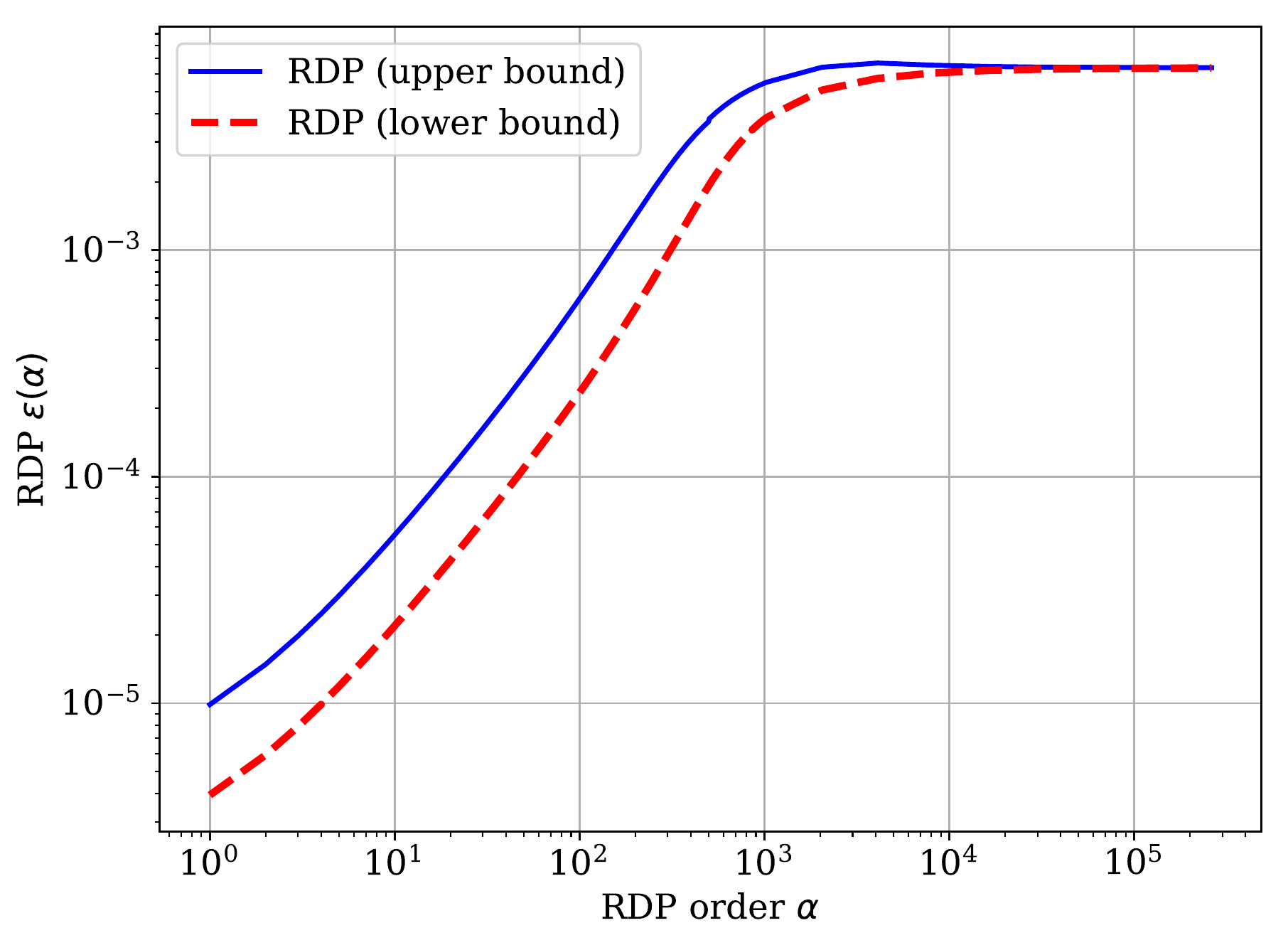}
		\caption{\scriptsize Subsampled Laplace with $b=0.5$.}
		\label{0e}
	\end{subfigure}
	\begin{subfigure}[t]{0.32\textwidth}
		\includegraphics[width=\textwidth]{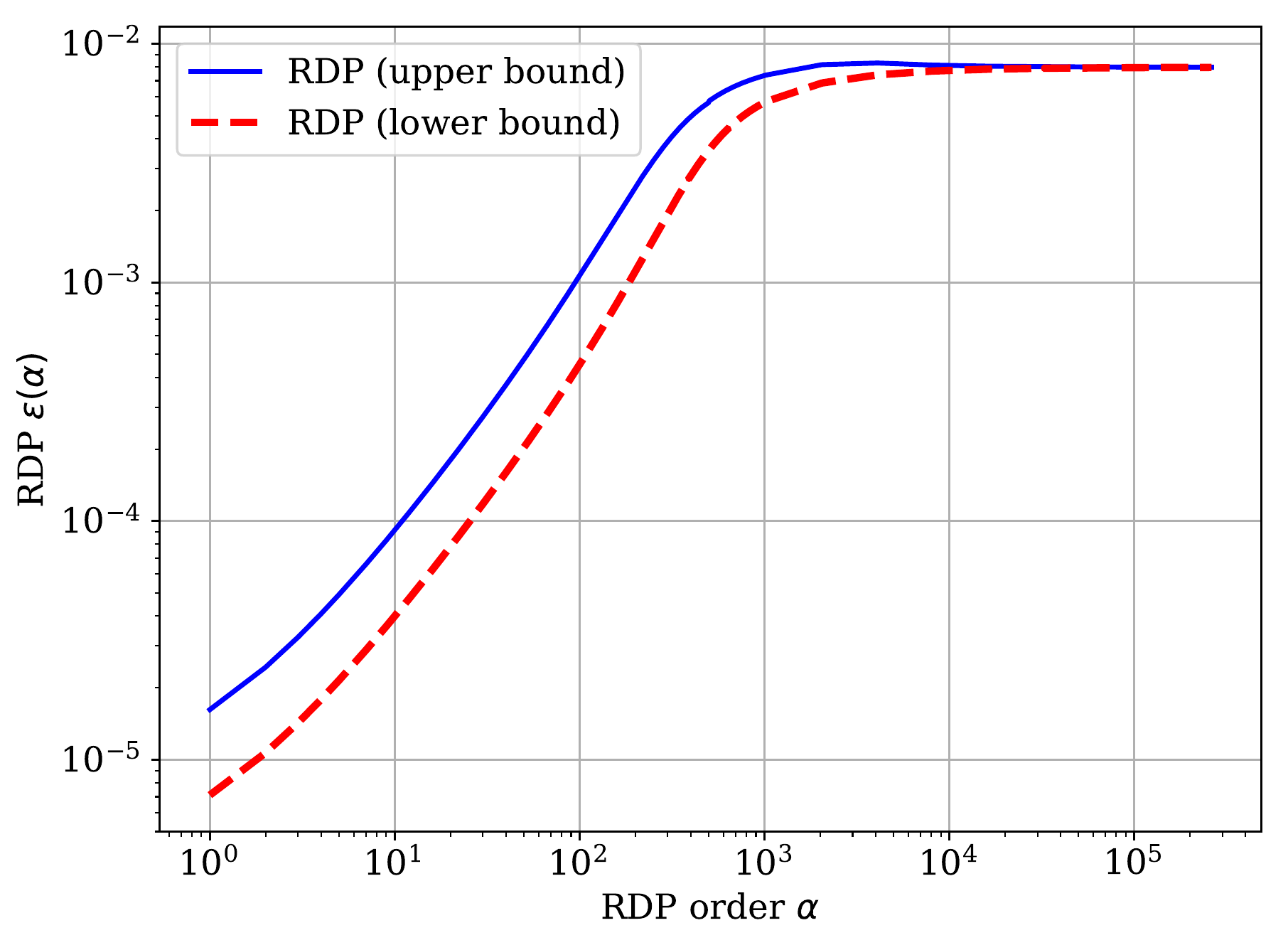}
		\caption{\scriptsize Subsampled Rand.\ Resp.\ with $p=0.9$}
		\label{0f}
	\end{subfigure}
	\caption{The RDP parameter ($\epsilon(\alpha))$ of the three subsampled mechanisms as a function of order $\alpha$, with subsampling rate $\gamma = 0.001$ in all the experiments. The top row illustrates the case where the base mechanism $\cM$ (before amplification using subsampling) is in a relatively high privacy regime (with $\epsilon \approx 0.5$) and the bottom row shows the low privacy regime with $\epsilon \approx 2$. RDP upper bound  obtained through Theorem~\ref{thm:main} is represented as the blue curve, and the corresponding lower bound obtained through Proposition~\ref{prop:lowerbound} is represented as the red dashed curve. For the Gaussian case, we also present the RDP bound obtained through the asymptotic Gaussian approximation idea explained in Appendix~\ref{app:gaussasym}.}\label{fig:composedRDP}
\end{figure}

Here $\sigma^2$ represents the variance of the Gaussian perturbation, $2 b^2$ the variance of the Laplace perturbation, and $p$ the probability of replying truthfully in randomized response.
We considered two groups of parameters $\sigma,b,p$ for the three base mechanisms $\cM$.
\begin{description}
	\item[High Privacy Regime:]  We set $\sigma=5$, $b=2$ and $p=0.6$. These correspond to $(0.2\sqrt{2\log(1.25/\delta)},\delta)$-DP, $(0.5,0)$-DP, and approximately $(0.41,0)$-DP for the Gaussian, Laplace, and Randomized response mechanisms, respectively, using the standard differential privacy calibration.
	\item[Low Privacy Regime:]  We set $\sigma=1$, $b=0.5$ and $p=0.9$. These correspond to $(\sqrt{2\log(1.25/\delta)},\delta)$-DP, $(2,0)$-DP, and approximately $(2.2,0)$-DP for the Gaussian, Laplace, and Randomized response mechanisms, respectively, using the standard differential privacy calibration.
\end{description}
The subsampling ratio $\gamma$ is taken to be $0.001$ for both regimes.

In Figure~\ref{fig:composedRDP}, we plot the upper and lower bounds (as well as asymptotic approximations whenever applicable) of RDP parameter $\epsilon'(\alpha)$ for the subsampled mechanism $\cM\circ\subsample$ as a function of $\alpha$. 
As we can see,
 the upper and lower bounds match up to a multiplicative constant for all the three mechanisms. There is a phase transition in the subsampled Gaussian case as we expect in both the upper and lower bound, which occurs at about $\gamma \alpha e^{\epsilon(\alpha)} < 1$. Note that our upper bound (the blue curve) matches the lower bound up to a multiplicative constant throughout in all regimes. For subsampled Gaussian mechanism in Plots~\ref{0a} and~\ref{0d}, the RDP parameter matches up to an (not visible in log scale) additive factor for large $\alpha$. 
The RDP parameter for subsampled Laplace and subsampled randomized response (in the second and third column) are both linear in $\alpha$ at the beginning, then they  flatten as $\epsilon(\alpha)$ approaches $\epsilon(\infty)$.

For the Gaussian mechanism we also plot an asymptotic approximation obtained under the assumption that the size of the input dataset grows $n \to \infty$ while the subsampling ratio $\gamma = m/n$ is kept constant.
In fact, we derive two asymptotic approximations: one in the case of ``good'' data and one for ``bad'' data.
The approximations and the definitions of ``good'' and ``bad'' data can be found in Appendix~\ref{app:gaussasym}.
The asymptotic Gaussian approximation with the ``bad'' data in Example~\ref{exp:asymp_approx_worst} matches almost exactly with lower bound up to the phase transition point both in the high- and low-privacy regimes.
The Gaussian approximation for the ``good'' data (with $n=100/\gamma$) is smaller than the lower bound, especially in the low-privacy regime, highlighting that we could potentially gain a lot by performing a dataset-dependent analysis.


\begin{figure}[t]
	\centering
	\begin{subfigure}[t]{0.32\textwidth}
		\includegraphics[width=\textwidth]{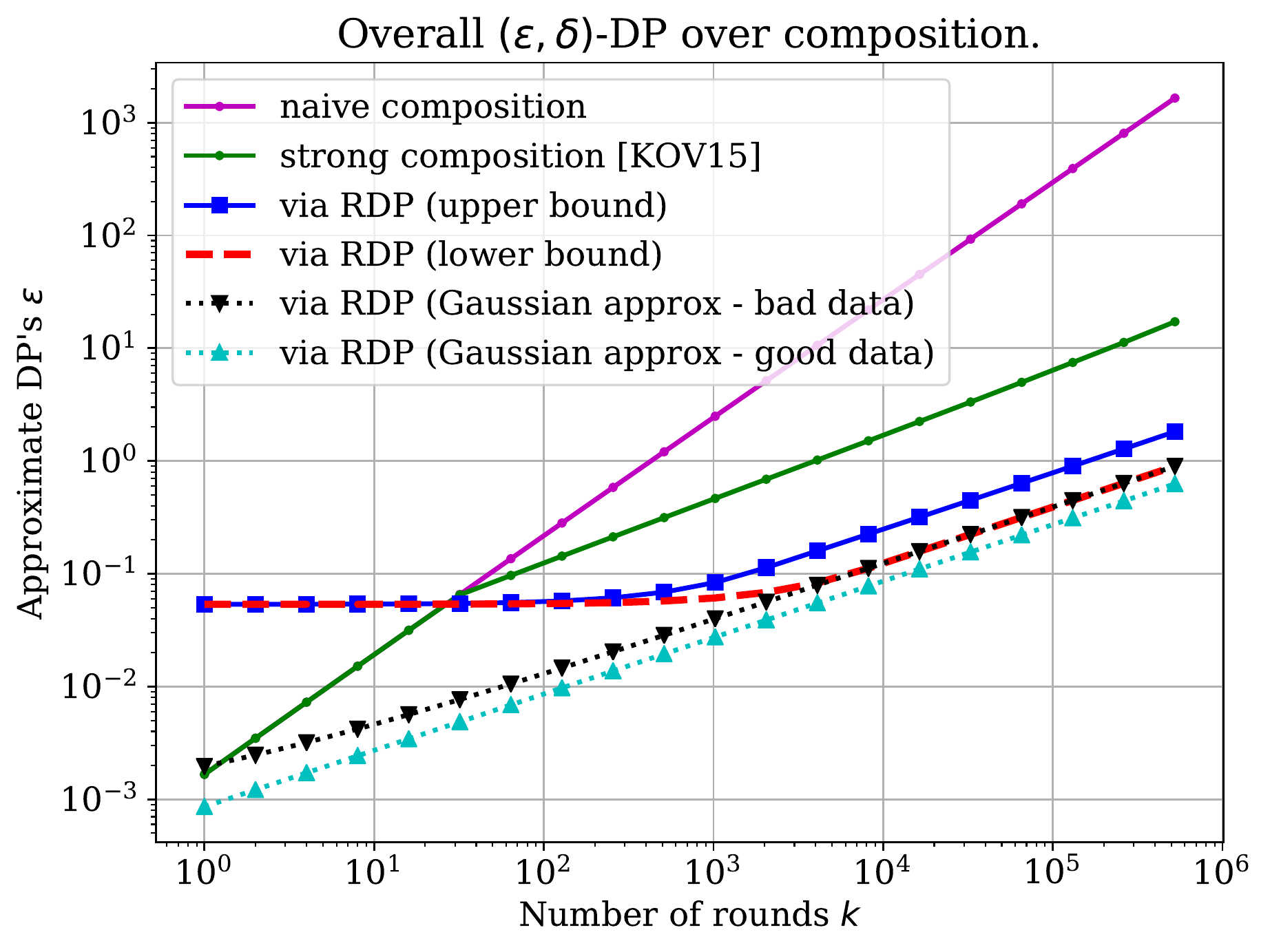}
		\caption{\scriptsize Subsampled Gaussian with $\sigma=5$.}
		\label{a}
	\end{subfigure}
	\begin{subfigure}[t]{0.32\textwidth}
		\includegraphics[width=\textwidth]{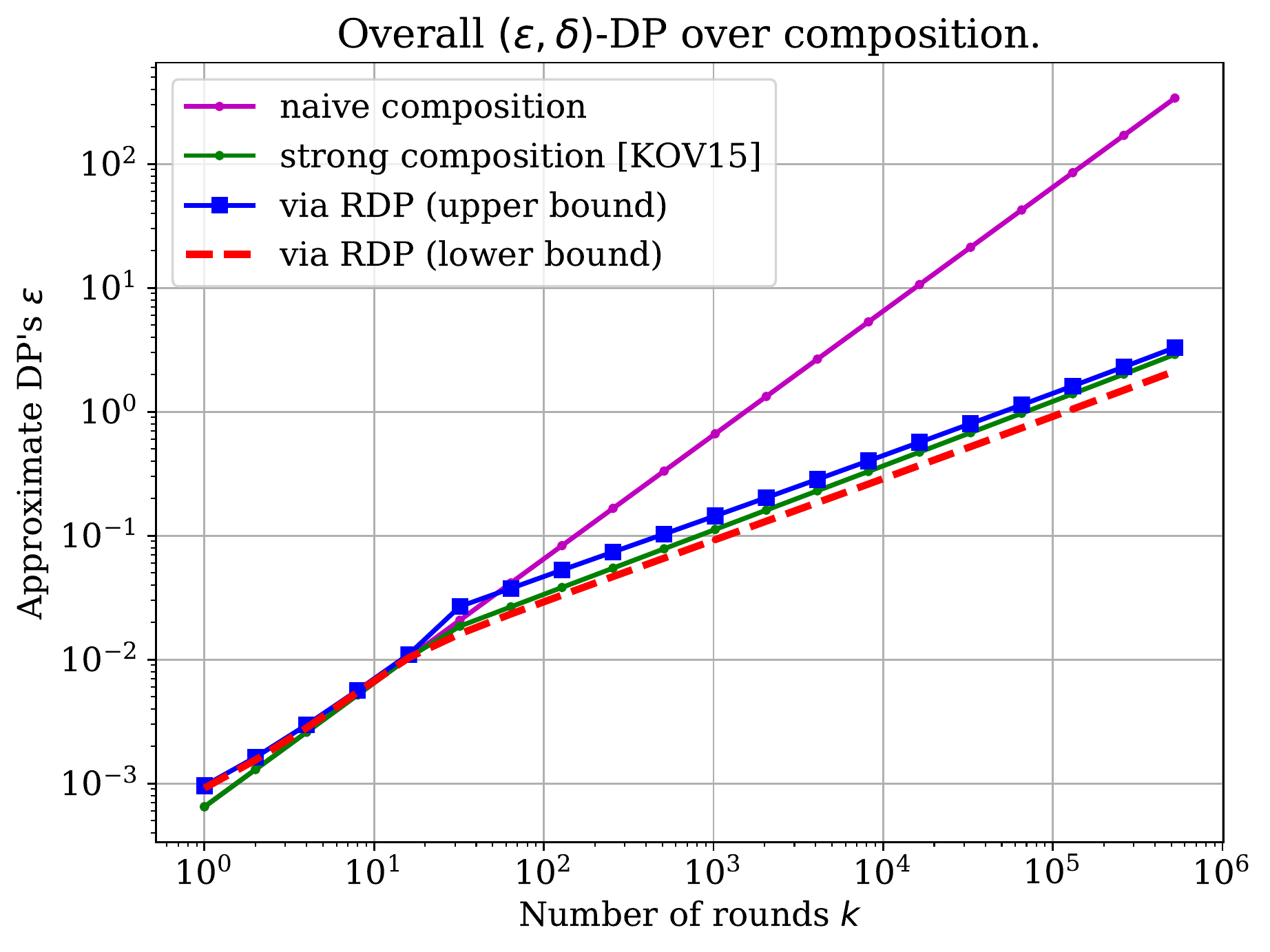}
		\caption{\scriptsize Subsampled Laplace with $b = 2$.}
		\label{b}
	\end{subfigure}
	\begin{subfigure}[t]{0.32\textwidth}
		\includegraphics[width=\textwidth]{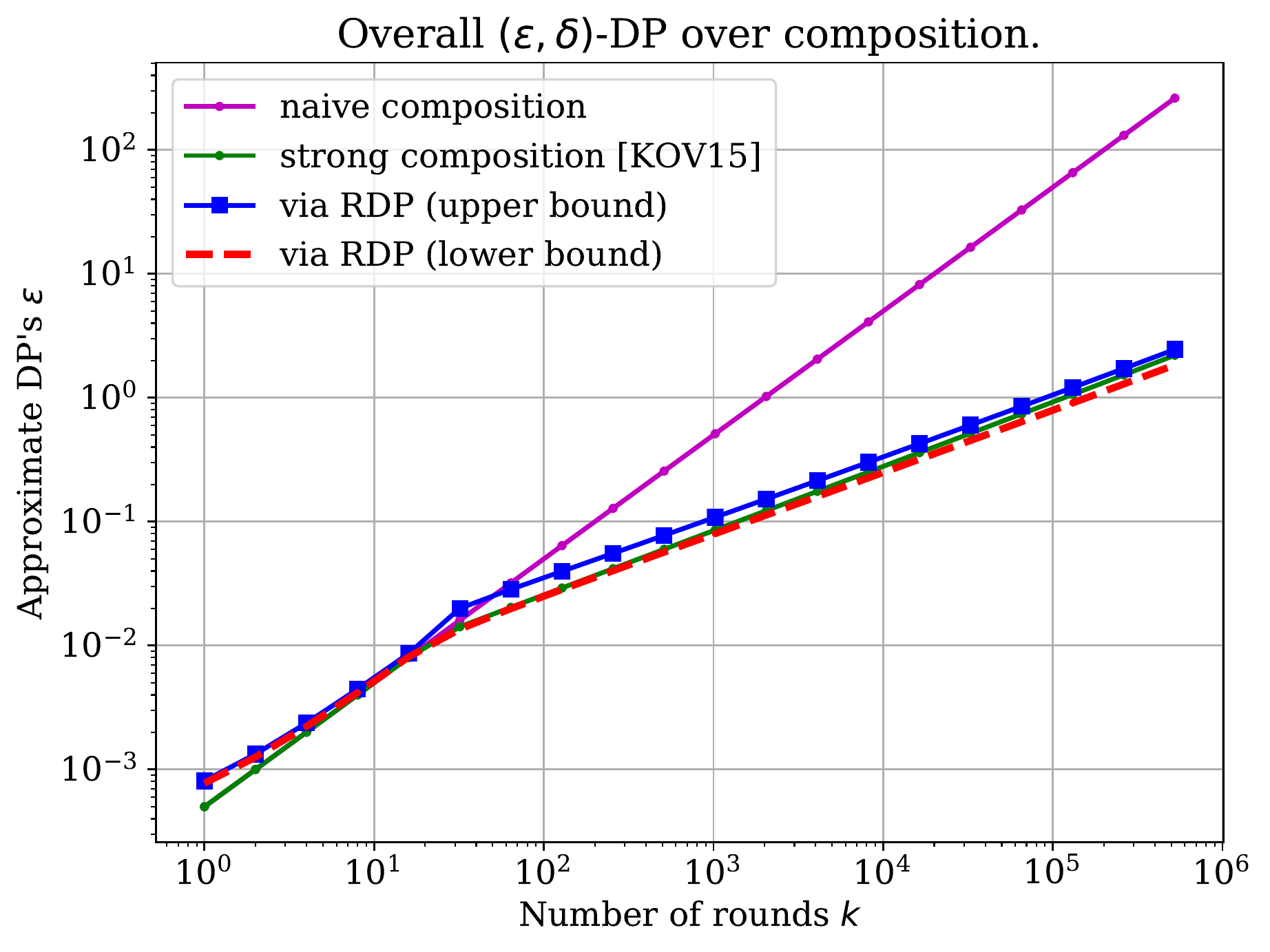}
		\caption{\scriptsize Subsampled Rand.\ Resp.\ with $p=0.6$.}
		\label{c}
	\end{subfigure}
	
	\begin{subfigure}[t]{0.32\textwidth}
		\includegraphics[width=\textwidth]{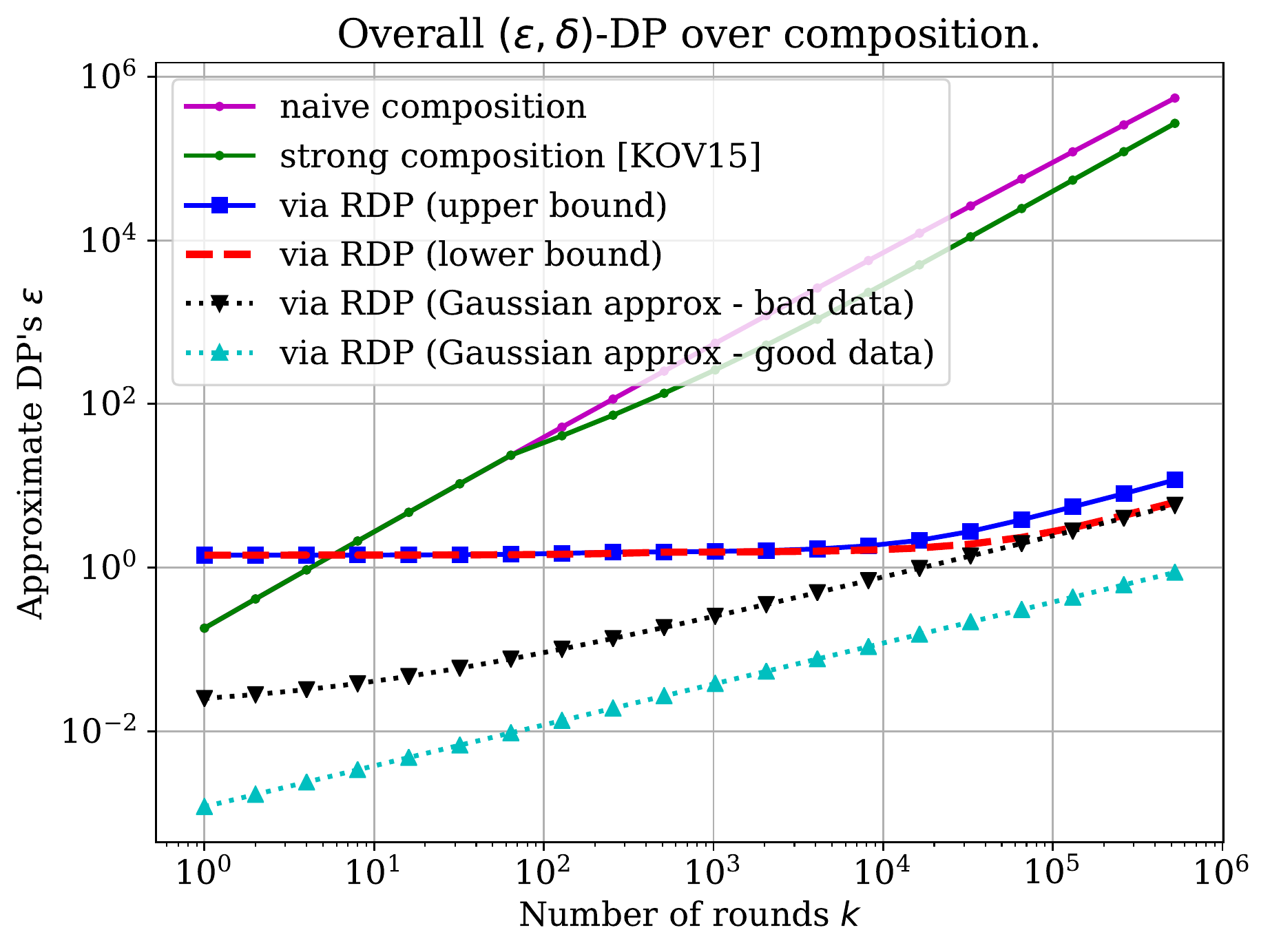}
		\caption{\scriptsize Subsampled Gaussian with $\sigma=0.5$.}
		\label{d}
	\end{subfigure}
	\begin{subfigure}[t]{0.32\textwidth}
		\includegraphics[width=\textwidth]{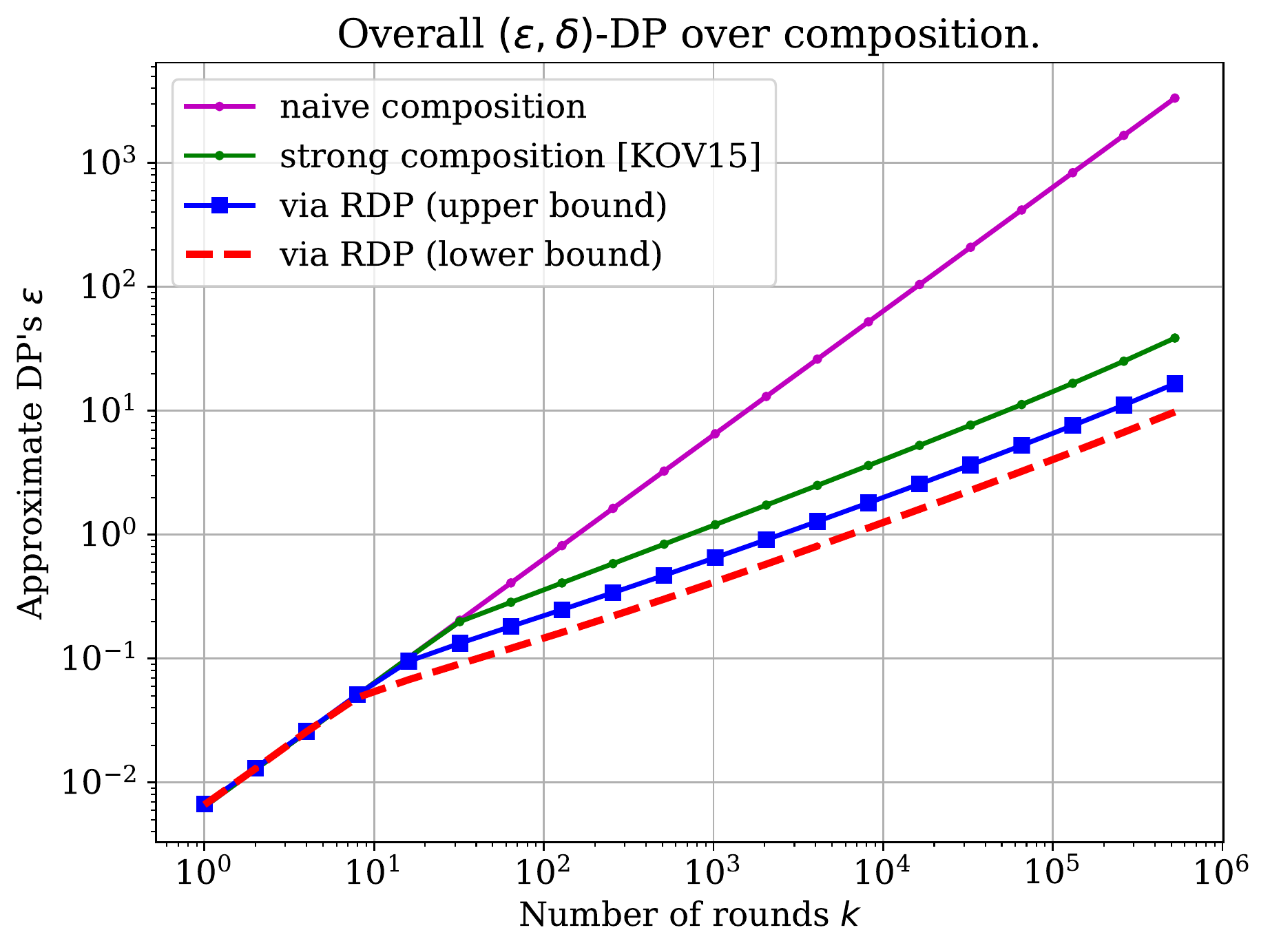}
		\caption{\scriptsize Subsampled Laplace with $b=0.5$.}
		\label{e}
	\end{subfigure}
	\begin{subfigure}[t]{0.32\textwidth}
		\includegraphics[width=\textwidth]{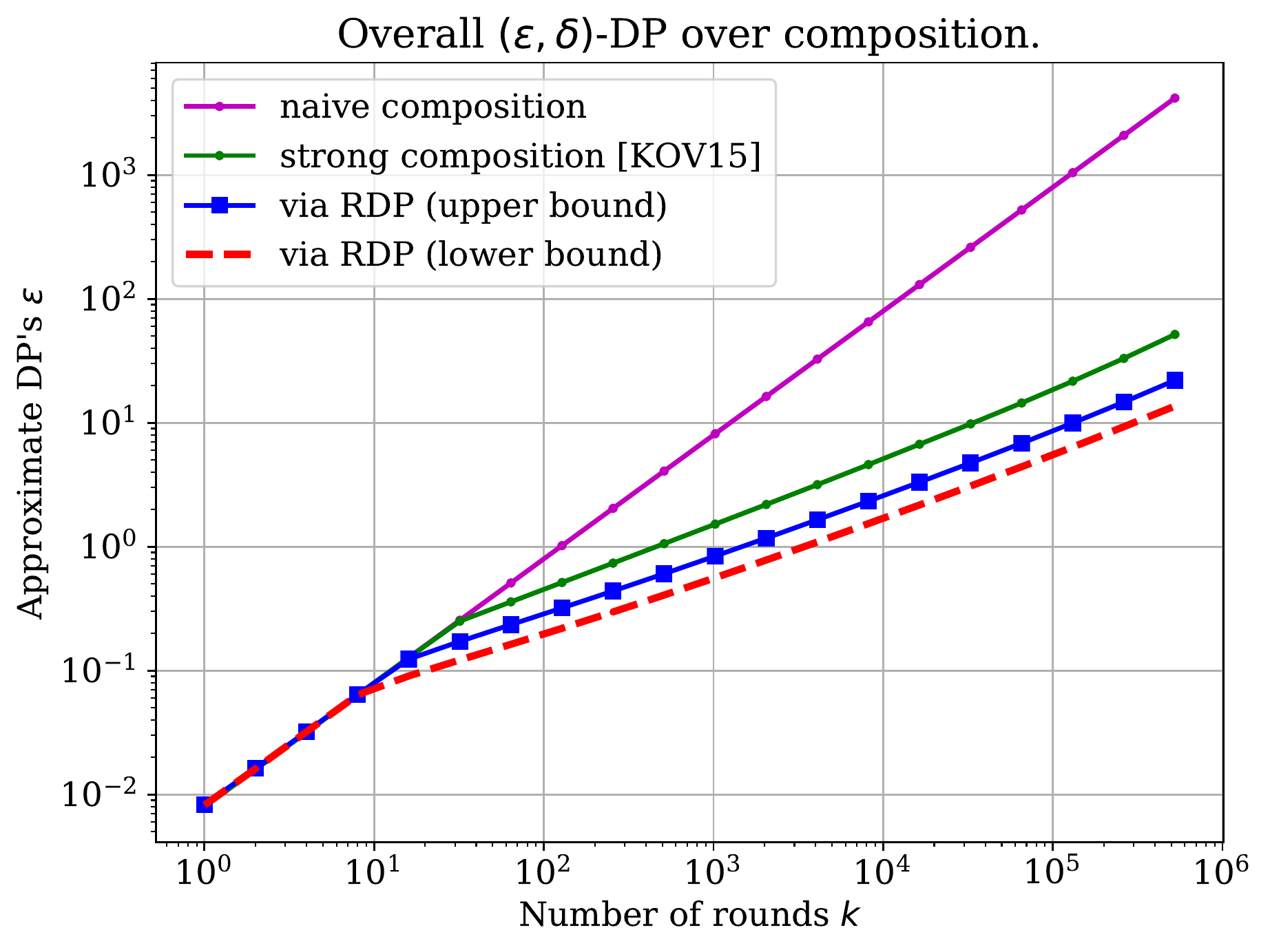}
		\caption{\scriptsize Subsampled Rand.\ Resp.\ with $p=0.9$}
		\label{f}
	\end{subfigure}
	\caption{Comparison of techniques for strong composition of $(\epsilon,\delta)$-DP over $600,000$ data accesses with three different subsampled mechanisms.  We plot $\epsilon$ as a function of the number of rounds of composition $k$ with $\delta=1e-8$ (note that smaller $\epsilon$ is better). The top row illustrates the case where the base mechanism $\cM$ (before amplification using subsampling) is in a relatively high privacy regime (with $\epsilon \approx 0.5$) and the bottom row shows the low privacy regime with $\epsilon \approx 2$.  We consider two baselines: the na\"ive composition that simply adds up $(\epsilon,\delta)$ and the strong composition is through the result of~\citep{kairouz2015composition} with an optimal choice of per-round $\delta$ parameter computed for every $k$. The blue curve is based on the composition applied to the RDP upper bound obtained through Theorem~\ref{thm:main}, and the red dashed curve is based on the composition applied to the lower bound on RDP obtained through Proposition~\ref{prop:lowerbound}.  For the Gaussian case, we also present the curves based on applying the composition on the RDP bound obtained through the Gaussian approximation idea explained in Appendix~\ref{app:gaussasym}.}\label{fig:subsampled_gaussian}
\end{figure}

In Figure~\ref{fig:subsampled_gaussian}, we plot the overall $(\epsilon,\delta)$-DP for $\delta=1e-8$ as we compose each of the three subsampled mechanisms for $600,000$ times.  The $\epsilon$ is obtained as a function of $\delta$ for each $k$ separately by calling the $\delta\Rightarrow\epsilon$ query in our analytical moments ccountant. Our results are compared to the algorithm-independent techniques for differential privacy including na\"ive composition and strong composition. The strong composition baseline is carefully calibrated for each $k$  by choosing an appropriate pair of $(\tilde{\epsilon},\tilde{\delta})$ for $\cM$ such that the overall $(\epsilon,\delta)$-DP guarantee that comes from composing $k$ rounds of  $\cM\circ\subsample$ using \citet{kairouz2015composition}  obeys that $\delta<1e-8$ and $\epsilon$ is minimized. Each round is described by the 
$\big(\log(1 + \gamma (e^{\tilde{\epsilon}}-1)),\gamma\tilde{\delta}\big)$-DP guarantee using the standard subsampling lemma (Lemma~\ref{lem:subsampling_approx}) and $\tilde{\epsilon}$ is obtained as a function of $\tilde{\delta}$ via \eqref{eq:eps_from_delta}. 

Not surprisingly, both our approach and strong composition give an $\sqrt{k}$ scaling while the na\"ive composition has an $O(k)$ scaling throughout. 
An interesting observation for the subsampled Gaussian mechanism is that the RDP approach initially performs worse than the na\"ive composition and strong composition with the standard subsampling lemma. Our RDP lower bound certifies that this is not due to an artifact of our analysis but rather a fundamental limitation of the approach that uses RDP to obtain $(\epsilon,\delta)$-DP guarantees. We believe this is a manifestation of the same phenomenon that leads to the sub-optimality of the classical analysis of the Gaussian mechanism \citep{balle2018improving}, which also relies on the conversion of a bound on the CGF of the privacy loss into an $(\epsilon,\delta)$-DP guarantee, and might be addressed using the necessary and sufficient condition for $(\epsilon,\delta)$-DP in terms of tail probabilities of the privacy loss random variable given in \citep[Theorem 5]{balle2018improving}.
Luckily, such an artifact does not affect the typical usage of RDP: as the number of rounds of composition continues to grow, we end up having about an order of magnitude smaller $\epsilon$ than the baseline approaches in the high privacy regime  (see Figure~\ref{a}) and five orders of magnitude smaller $\epsilon$ in the low privacy regime (see Figure~\ref{d}).

The results for composing subsampled Laplace mechanisms and subsampled randomized response mechanisms are shown in Figures~\ref{b},~\ref{c},~\ref{e}, and~\ref{f}. Unlike the subsampled Gaussian case, the RDP-based approach achieves about the same or better $\epsilon$ bound for all $k$ when compared to what can be obtained using a subsampling lemma and strong composition.




\section{Conclusion}

In this paper, we have studied the effect of subsampling (without replacement) in amplifying R\'enyi differential privacy (RDP). Specifically, we established a tight upper and lower bound for the RDP parameter for the randomized algorithm  $\cM\circ\subsample$ that first subsamples the data set then applies $\cM$ to the subsample, in terms of the RDP parameter of $\cM$.  
Our analysis also reveals interesting theoretical insight into the connection of subsampling to a linearized privacy random variable, higher order discrete differences of moment generating functions, as well as  a ternary version of Pearson-Vajda divergence that appears fundamental in understanding and analyzing the effect of subsampling.  
In addition, we designed a data structure called \emph{analytical moments accountant} which composes RDP for randomized algorithm (including subsampled ones) in symbolic forms and allows efficiently conversion of RDP to $(\epsilon,\delta)$-DP for any $\delta$ (or $\epsilon$) of choice. These results substantially expands the scope of the mechanisms with RDP guarantees to cover subsampled versions of Gaussian mechanism, Laplace mechanism, Randomized Responses, posterior sampling and so on, which facilitates flexible differentially private algorithm design. We compared our approach to the standard approaches that use subsampling lemma on $(\epsilon,\delta)$-DP directly and then applies strong composition, and in our experiments we notice an order of magnitude improvement in the privacy parameters with our bounds when we compose the subsampled Gaussian mechanism over multiple rounds.

Future work includes applying this technique to more advanced mechanisms for differentially private training of neural networks, addressing the data-dependent per-instance RDP for subsampled mechanisms, connecting the problem  more tightly with statistical procedures that uses subsampling/resampling as key components such as {\em bootstrap} and {\em jackknife}, as well as combining the new approach with subsampling-based sublinear algorithms for exploratory data analysis.

\section*{Acknowledgment}
The authors thank Ilya Mironov and Kunal Talwar for helpful discussions and the clarification of their proof of Lemma 3 in~\citep{abadi2016deep}.

\bibliographystyle{apa-good}
\bibliography{FreeDP}

\appendix

\section{Composition of Differentially Private Mechanisms} \label{app:dp}
Composition theorems for differential privacy allow a modular design of privacy preserving mechanisms based on mechanisms for simpler sub tasks:

\begin{theorem}[Na\"ive composition,~\citet{dwork2006our}]\label{thm:composition1}
	A mechanism that permits $k$ adaptive interactions with mechanisms that preserves $(\epsilon,\delta)$-differential privacy (and does not access the database otherwise) ensures $(k\epsilon, k\delta)$-differential privacy.
\end{theorem}

A stronger composition is also possible as shown by~\citet{dwork2010boosting}.
\begin{theorem}[Strong composition,~\citet{dwork2010boosting}]\label{thm:composition2}
	Let $\epsilon,\delta,\delta^\ast>0$ and $\epsilon \leq 1$. A mechanism that permits $k$ adaptive interactions with mechanisms that preserves $(\epsilon,\delta)$-differential privacy 
	ensures $(\epsilon\sqrt{2k\ln(1/\delta^\ast)}+ 2 k\epsilon^2, k\delta+\delta^\ast)$-differential privacy.
\end{theorem}
\citet{kairouz2015composition} recently gave an optimal composition theorem for differential privacy, which provides an exact characterization of the best privacy parameters that can be guaranteed when composing a number of $(\epsilon,\delta)$-differentially private mechanisms. Unfortunately, the resulting optimal composition bound is quite complex to state exactly, and indeed is even \#P-complete to compute exactly when composing mechanisms with different $(\epsilon_i, \delta_i)$ parameters~\citep{murtagh2016complexity}.

\section{Proofs and Missing Details from Section~\ref{sec:amp} }\label{app:subsampling}
In this section, we fill in the missing details and proofs from Section~\ref{sec:amp}. We first define a few quantities needed to establish our results. 

\noindent\textbf{Pearson-Vajda Divergence and the Moments of Linearized Privacy Random Variable.}
The Pearson-Vajda Divergence (or $|\chi|^{\alpha}$-divergence) of order $\alpha$ is defined as follows \citep{vajda1973chialpha}:
\begin{equation} \label{eqn:bin}
D_{|\chi|^{\alpha}}(p\|q)  := \E_q\left[ \left| \frac{p}{q} -1  \right|^\alpha \right].
\end{equation}
This is closely related to the moment of the privacy random variable in that $(p/q-1)$ is the linearized version of $\log(p/q)$. More interestingly, the $\alpha$th moment of the privacy random variable is the $\alpha$th derivate of the MGF evaluated at $0$:
$$\E[\log(p/q)^\alpha] = \frac{\partial^\alpha}{\partial t^\alpha}[e^{K_\cM(t)}] (0),$$
while at least for the even order, the $|\chi|^\alpha$-divergence is the $\alpha$th order \emph{forward finite difference} of the MGF evaluated at $0$:
\begin{equation}\label{eq:discrete_mgf}
\E[(p/q-1)^\alpha]  = \Delta^{(\alpha)}[ e^{K_\cM(\cdot)} ] (0).
\end{equation}
In the above expression, the  $\alpha$th order  {\em forward difference operator} $\Delta^{(\alpha)}$ is defined recursively with
\begin{align} \label{eqn:fwd}
\Delta^{(\alpha)}  :=  \underbrace{\Delta\circ ...\circ\Delta}_{\alpha\text{-times}},
\end{align}
where $\Delta$ denote the first order forward difference operator such that $\Delta[f](x) =  f(x+1)-f(x)$ for any function $f: \R\rightarrow \R$. See Appendix~\ref{app:discrete_calculus} for more information on $\Delta^{(\alpha)}$ and its connection to binomial numbers.

\subsection{A Sketch of the Proof of Theorem~\ref{thm:main}} \label{app:mainproof}
In this section, we present a sketch of the proof of our main theorem.  The arguments are divided into three parts.  In the first part, we define a new family of privacy definitions called \emph{ternary-$|\chi|^\alpha$-differential privacy} and show that it handles subsampling naturally.  In the second part, we bound the R\'enyi DP using the ternary-$|\chi|^\alpha$-differential privacy and apply their subsampling lemma. In the third part, we propose several different ways of converting the expression stated as ternary-$|\chi|^\alpha$-differential privacy back to that of RDP, hence giving rise to the stated results in the remarks following Theorem~\ref{thm:main}.

\noindent\textbf{Part 1: Ternary-$|\chi|^\alpha$-divergence and Natural Subsampling.}
Ternary-$|\chi|^\alpha$-divergence is a novel quantity that measures the discrepancy of three distributions instead of two.  Let $p,q,r$ be three probability distributions\footnote{We think of $p,q,r$ as the distributions $\cM\circ\mathsf{subsample}(X),\cM\circ\mathsf{subsample}(X'),\cM\circ\mathsf{subsample}(X'')$, respectively, for mutually adjacent datasets $X,X',X''$.}, we define
$$
D_{|\chi|^{\alpha}}(p,q\|r)  := \E_r\left[ \left| \frac{p-q}{r} \right|^\alpha \right].  
$$
Using, this ternary-$|\chi|^\alpha$-divergence notion, we define $\zeta$-ternary-$|\chi|^\alpha$-differential privacy as follows. Analogously with RDP where we considered $\epsilon$ as a function of $\alpha$, we consider $\zeta$ as a function of $\alpha$.
\begin{definition}[Ternary-$|\chi|^\alpha$-differential privacy]
	We say that a randomized mechanism $\cM$ is $\zeta$-ternary-$|\chi|^\alpha$-DP if for all $\alpha \geq 1$:
	$$\sup_{X,X',X'' \mathrm{\, mutually \, adjacent}}\Big( D_{|\chi|^{\alpha}}(\cM(X),\cM(X')\|\cM(X'')) \Big)^{1/\alpha} \leq \zeta(\alpha).$$	
\end{definition}
Here, the \emph{mutually adjacent} condition means $d(X,X'), d(X', X''), d(X,X'') \leq 1$, and $\zeta(\alpha)$ is a function from $\R^+$ to $\R^+$. Note that the above definition is a general case of the following binary-$|\chi|^\alpha$-differential privacy definition that works with the standard Person-Vajda $|\chi|^\alpha$-divergences (as defined in~\eqref{eqn:bin}).
\begin{definition}[Binary-$|\chi|^\alpha$-differential privacy]
	We say that a randomized mechanism $\cM$ is $\xi$-binary-$|\chi|^\alpha$-DP if for all $\alpha \geq 1$:
	$$\sup_{X,X' : d(X,X') \leq 1}\Big( D_{|\chi|^{\alpha}}(\cM(X)\|\cM(X')) \Big)^{1/\alpha} \leq \xi(\alpha).$$	
\end{definition}
Again, $\xi(\alpha)$ is a function from $\R^+$ to $\R^+$.

As we described earlier, this notion of privacy shares many features of RDP and could have independent interest.  It subsumes $(\epsilon,0)$-DP  (for $\alpha\rightarrow \infty$) and implies an entire family of $(\epsilon(\delta),\delta)$-DP through Markov's inequality. We provide additional details on this point in Appendix~\ref{app:properties_of_chi_DP}.  

For our ternary-$|\chi|^\alpha$-differential privacy, what makes it stand out relative to R\'enyi DP is how it allows privacy amplification to occur in an extremely clean fashion, as the following proposition states:
\begin{proposition}[Subsampling Lemma for Ternary-$|\chi|^\alpha$-DP]\label{prop:subsample_ternary}
	Let a mechanism $\cM$ obey $\zeta$-ternary-$|\chi|^\alpha$-DP, then the algorithm $\cM\circ \subsample$ obeys  $\gamma\zeta$-ternary-$|\chi|^\alpha$-DP.
\end{proposition}
The entire proof is presented in Appendix~\ref{app:part1}. The key idea involves using conditioning on subsampling events,  constructing dummy random variables to match up each of these events, and the use of Jensen's inequality to convert the intractable ternary-$|\chi|^\alpha$-DP of a mixture distribution to that of three simple distributions that come from mutually adjacent datasets.

\noindent\textbf{Part 2: Bounding RDP with Ternary-$|\chi|^\alpha$-DP.}
We will now show that (a transformation of) the quantity of interest --- RDP of the subsampled mechanism --- can be expressed as a linear combination of a sequence of binary-$|\chi|^\alpha$-DP parameters $\xi(\alpha)$ for integer $\alpha=2,3,...$ through Newton's series expansion of the moment generating function:
\begin{align}
\E_q \left[\left(\frac{p}{q}\right)^{\alpha}\right] = 1 + {\alpha\choose 1} \E_q\left[\frac{p}{q}-1\right] + \sum_{j=2}^{\alpha} {\alpha\choose j} \E_q\left[\left(\frac{p}{q}-1\right)^j\right].\label{eq:newton_series}
\end{align}
Observe that
$\E_q\left[\frac{p}{q}-1\right]=0$, so it suffices to bound $\E_q\left[\left(\frac{p}{q}-1\right)^j\right]$ for $j\geq 2$.

Note that $\frac{p}{q}-1$ is a special case of $(p-q)/r$ with $q=r$, therefore,
$$\max_{p,q} \, \E_q \left [ \left(\frac{p -q}{q}\right)^j \right ] \leq \max_{p,q,r} \, \E_r \left [ \left(\frac{p -q}{r}\right)^j \right ] \leq \max_{p,q,r} D_{|\chi|^{j}}(p,q\|r).$$
The same holds if we write $\cM' = \cM\circ\mathsf{subsample}$ and restrict the maximum on the left to $p = \cM'(X)$ and $q = \cM'(X')$ with $X$, $X'$ adjacent, and the maximum on the right to $p = \cM'(X)$, $q = \cM'(X')$ and $r = \cM'(X')$ with mutually adjacent $X$, $X'$ and $X''$.
For the subsampled mechanism, the right-hand side of the above equation can be bounded by Proposition~\ref{prop:subsample_ternary}. Putting these together, we can bound~\eqref{eq:newton_series} as
\begin{align*}
\E_q \left[\left(\frac{p}{q}\right)^{\alpha}\right] \leq 1 + \sum_{j=2}^{\alpha} {\alpha\choose j} \gamma^j \zeta(j)^j,
\end{align*}
where mechanism $\cM$ satisfies $\zeta$-ternary-$|\chi|^\alpha$-DP and $p,q$ denote the distributions $\cM\circ\mathsf{subsample}(X),\cM\circ\mathsf{subsample}(X')$, respectively, for adjacent datasets $X,X'$. Using this result along with the definition of R\'enyi differential privacy (from Definition~\ref{def:RDP}) implies the RDP parameter following bound,
\begin{equation}\label{eq:amplify_renyi_with_zeta}
\epsilon_{\cM\circ\subsample}(\alpha) \leq \frac{1}{\alpha-1}\log \Big( 1 + \sum_{j=2}^{\alpha} {\alpha\choose j} \gamma^j \zeta(j)^j\Big),
\end{equation}
\noindent\textbf{Part 3: Bounding Ternary-$|\chi|^\alpha$-DP using RDP.}
It remains to bound $\zeta(j)^j := \sup_{p,q,r} \E_r[ \frac{|p-q|^j}{r^j}]$ using RDP. We provide several ways of doing so and plugging them into \eqref{eq:amplify_renyi_with_zeta} show how the various terms in the bound of Theorem~\ref{thm:main} arise. Missing proofs are presented in Appendix~\ref{app:part3}.

\begin{list}{{\bf (\alph{enumi})}}{\usecounter{enumi}
		\setlength{\leftmargin}{11pt}
		\setlength{\listparindent}{0pt}
		\setlength{\parsep}{0pt}}
	\item\textbf{The $4(e^{\epsilon(2)}-1)$ Term.} To begin with, we show that the binary-$|\chi|^\alpha$-DP and ternary-$|\chi|^\alpha$-DP are  equivalent up to a constant of $4$.
	\begin{lemma}\label{lem:ternary2binary}
		If a randomized mechanism $\cM$ is $\xi$-binary-$|\chi|^\alpha$-DP, then it is $\zeta$-ternary-$|\chi|^\alpha$-DP for some $\zeta$ satisfying $\xi(\alpha)^\alpha \leq \zeta(\alpha)^\alpha \leq 4\xi(\alpha)^\alpha$.
	\end{lemma}
	For the special case of $j=2$, we have 
	$$\E_q[|p/q-1|^2]  =   \E_q[(p/q)^2] - 2 \E_q[p/q]  + 1  =e^{\epsilon(2)}-1.$$ Using the bound from Lemma~\ref{lem:ternary2binary} relating the binary and ternary-$|\chi|^\alpha$-DP, gives that $\zeta(2) \leq 4(e^{\epsilon(2)}-1)$.
	
	\item\textbf{The $e^{(j-1)\epsilon(j)} \min\{2, (e^{\epsilon(\infty)}-1)^{j}\}$ Term.}  Now, we provide a bound for $j\geq 2$. We start with the following simple lemma.
	\begin{lemma}\label{lem:triangular_of_diff_pow}
		Let $X,Y$ be nonnegative random variables,  for any $j\geq 1$
		$$\E[ |X - Y|^{j}] \leq \E [X^j] + \E[Y^j].$$
	\end{lemma}
	This ``triangular inequality''-like result exploits the nonnegativity of $X,Y$ and captures the intrinsic cancellations of the $2^j$ terms of a Binomial expansion. If we do not have non-negativity, the standard expansion will have a $2^j$ factor rather than $2$ (see e.g., Proposition 3.2 of \cite{bobkov2016r}).  
	
	An alternative bound that is tighter in cases when $X$ and $Y$ is related to each other with a multiplicative bound.  Note that this bound is only going to be useful when $\cM$ has a bounded $\epsilon(\infty)$, such as when $\cM$ satisfies $(\epsilon,0)$-DP guarantee.
	\begin{lemma}\label{lem:pure_dp_bound}
		Let $X,Y$ be nonnegative random variables and with probability $1$, $e^{-\varepsilon} Y \leq X \leq e^{\varepsilon} Y $. Then for any $j\geq 1$
		$$\E[ |X - Y|^{j}] \leq \E[Y^j](e^{\varepsilon} - 1)^j.$$
	\end{lemma}
	Take $X = p/r$ and $Y=q/r$. Applying Lemma~\ref{lem:triangular_of_diff_pow} gives $\zeta(j) \leq 2 e^{(j-1)\epsilon(j)}$. Using Lemma~\ref{lem:pure_dp_bound} instead with $\varepsilon = \epsilon(\infty)$ provided by the mechanism $\cM$, we have $\zeta(j) \leq e^{(j-1)\epsilon(j)} (e^{\epsilon(\infty)}-1)^{j}$. Using these bounds together, we get the overall bound of, 
	$$\zeta(j) \leq e^{(j-1)\epsilon(j)}  \min\{2, (e^{\epsilon(\infty)}-1)^{j}\}.$$
	Note that at $j=2$, $e^{(j-1)\epsilon(j)} \min\{2, (e^{\epsilon(\infty)}-1)^{j}\}$ simplifies to $e^{\epsilon(2)} \min\{2, (e^{\epsilon(\infty)}-1)^{2}\}$.
\end{list}

\subsection{Proof of the Subsampling Lemma for Ternary-$|\chi|^\alpha$-DP} \label{app:part1}
In this section, we prove Proposition~\ref{prop:subsample_ternary}. The proof uses the following simple lemma.
\begin{lemma}\label{lem:convex}
	Bivariate function $f(x,y) =  x^{j} / y^{j-1}$ is jointly convex on $\R_+^2$ for $j >1$.
\end{lemma}
\begin{proof}
	Note that the function is continuously differentiable on $\R_+^2$. The two eigenvalues of the Hessian matrix 
	$$0 \quad\text{ and }\quad (j^2-j) \frac{x^j }{y^{j+1}} \left(1+ \frac{y^2}{x^2}\right)$$
	and both are nonnegative in the first quadrant. 
\end{proof}

\begin{proposition}[Proposition~\ref{prop:subsample_ternary} Restated]
	Let a mechanism $\cM$ obey $\zeta$-ternary-$|\chi|^\alpha$-DP, then the algorithm $\cM\circ \subsample$ obeys  $\gamma\zeta$-ternary-$|\chi|^\alpha$-DP.
\end{proposition}
\begin{proof}
	If three datasets $X,X',X''$ of size $n$ are mutually adjacent, they must differ on the same data point (w.l.o.g., let it be the $n$th), and the remaining $n-1$ data points are the same. Let $p,q,r$ denote the distributions $\cM\circ\mathsf{subsample}(X),\cM\circ\mathsf{subsample}(X'),\cM\circ\mathsf{subsample}(X'')$, respectively. 
	
	Let $E$ be the event such that the subsample includes the $n$th item (and $E^c$ be complement event), we have
	\begin{align*}
	p =  \gamma p(\cdot | E) + (1-\gamma) p(\cdot | E^c) \\
	q = \gamma q(\cdot | E) + (1-\gamma) q(\cdot | E^c).
	\end{align*}
	and by construction, $p(\cdot | E^c)  = q(\cdot | E^c)$.
	
	Substituting the observation into the ternary-$|\chi|^j$-divergence, we get $\gamma^j$ to show up.
	\begin{align}
	D_{|\chi|^{j}}(p,q \| r ) &=
	\E_r\left[\Big(\frac{|p-q|}{r}\Big)^{j}\right]  =  \gamma^j  \E_r\left[ \Big(\frac{|p(\cdot | E) - q(\cdot | E)|  }{ r}\Big)^j\right] \notag \\ &= \gamma^j D_{|\chi|^{j}}(p(\cdot | E),q(\cdot | E) \| r ). \label{eq:amplify_ternary}
	\end{align}
	
	Note that $p(\cdot | E), q(\cdot | E)$ and $r$ are mixture distributions with combinatorially many mixing components. 
	
	Let $J$ be a random subset of size $\gamma n$ chosen by the $\mathsf{subsample}$ operator.  In addition, we define an auxiliary dummy variable $i \sim \text{Unif}({1,...,\gamma n})$. Let $i$ be independent to everything else, so it is clear that 
	$r(\theta|J)  =  r(\theta|J, i)$.
	In other words,
	$$ r(\theta) =  \E_{J,i}[q(\theta|J, i)]  =  \frac{1}{\gamma n {n \choose \gamma n}} \sum_{J\subset[n],i\in[\gamma n]} r(\theta|J).$$
	
	Now, define functions $g$ and $g'$ on index set $J,i$ such that: 
	\begin{align*}
	g(J,i)  =  \begin{cases}
	p(\theta|J)& \text{ if } n\in J\\
	p(\theta|J\cup\{n\}\backslash J[i])& \text{ otherwise, }
	\end{cases}
	&&
	g'(J,i)  =  \begin{cases}
	q(\theta|J) & \text{ if } n\in J\\
	q(\theta|J\cup\{n\}\backslash J[i]) & \text{ otherwise. }
	\end{cases}
	\end{align*}
	Check that $p(\theta|E) = \E_{J,i} g(J,i)$ and $q(\theta|E) = \E_{J,i} g'(J,i)$.
	
	The above definitions and the introduction of the dummy random variable $i$ may seem mysterious. Let us explain the rationale behind them. Note that mixture distributions $p(\theta|E), q(\theta|E)$ have a different number of mixture components comparing to $q(\theta)$.  $q(\theta)$ has ${n \choose \gamma n}$ components while $p(\theta|E)$ and $q(\theta|E)$ only have ${n-1\choose \gamma n -1}$ components due to the conditioning on the event $E$ that fixes the differing (say the $n$th) datapoint in the sampled set. 
	
	The dummy random variable $i$ allows us to define a new $\sigma$-field to redundantly represent both subsampling over $[n-1]$ and $[n]$ under the same uniform probability measure while establishing a one-to-one mapping between pairs of events such that the corresponding index of the subsample differs by only one datapoint. 
	
	This trick allows us to write: 
	\begin{align}
	\E_q\left(\frac{|p(\theta|E)-q(\theta|E)|}{q(\theta)}\right)^j &= \int \frac{\left[p(\theta|E)-q(\theta|E)\right]^j}{q(\theta)^{j-1}} d\theta\nonumber\\
	&\explain{\leq}{\text{Jensen}}\int \E_{J,i}\left[  \frac{|g(J,i)-g'(J,i)|^{j}}{q(\theta|J)^{j-1}}\right] d\theta\nonumber\\
	&\explain{=}{\text{Fubini}} \E_{J,i} \E_{q} \left[\left(\frac{|g(J,i)-g'(J,i)|}{q(\theta|J)}\right)^j \;\middle|\; J,i\right] \leq \zeta(j)^j. \label{eq:subsampling_deriv1}
	\end{align}
	The second but last line uses Jensen's inequality and Lemma~\ref{lem:convex}, which proves the joint convexity of function $x^{j}/y(j-1)$ on $\R_+^2$. In the last line, we exchange the order of the integral, from which we get the expression for the ternary DP directly. Combining~\eqref{eq:amplify_ternary} with~\eqref{eq:subsampling_deriv1} gives the claimed result because the definitions of $g$ and $g'$ ensure that each inner expectation is a ternary Liese--Vajda divergence of the original mechanism on a triple of mutually adjacent datasets.
\end{proof}

\subsection{Missing Proofs on Bounding Ternary-$|\chi|^\alpha$-DP using RDP} \label{app:part3}
\begin{lemma}[Lemma~\ref{lem:ternary2binary} Restated]
	If a randomized mechanism $\cM$ is $\xi$-binary-$|\chi|^\alpha$-DP, then it is $\zeta$-ternary-$|\chi|^\alpha$-DP for some $\zeta$ satisfying $\xi(\alpha)^\alpha \leq \zeta(\alpha)^\alpha \leq 4\xi(\alpha)^\alpha$.
\end{lemma}
\begin{proof}
	The first inequality follows trivially by definition. We now prove the second. Let $p,q,r$ be three probability distributions.
	Consider four events:
	$$
	\{x| p\geq q, q\geq r\},    \{x| p\geq q, q <  r\},  \{x| p< q, p\geq r\}, \{x| p< q, p\geq r\}
	$$
	Under the first event $|p-q|^j/r^{j-1}  = (p-q)^j/r^{j-1} \leq (p-r)^j/r^{j-1}$.  Under the second event $ |p-q|^j/r^{j-1} \leq (p-q)^j/q^j $.Similarly, under the third and fourth event, $|p-q|^j/r^{j-1}$ is bounded by $(q-r)^j/r^{j-1}$ and $(q-p)^j/p^{j-1}$ respectively. It then follows that:
	\begin{align*}
	&	\E_r[ |p-q|^j/r^j]   \\
	=&  \E_r[ |p-q|^j/r^{j}  \mathbf{1}_{\{E_1\}} ]  +   \E_r[ |p-q|^j/r^{j}  \mathbf{1}_{\{E_2\}}  ] +  \E_r[  |p-q|^j/r^{j} \mathbf{1}_{\{E_3\}} ] +  \E_r[  |p-q|^j/r^{j} \mathbf{1}_{\{E_4\}} ]\\
	\leq&   \E_r[ |p-r|^j/r^{j}  \mathbf{1}_{\{E_1\}}  ]  +  \E_q[ |p-q|^j/q^{j}  \mathbf{1}_{\{E_2\}}  ]  +   \E_r[ |q-r|^j/r^{j}  \mathbf{1}_{\{E_3\}} ]  +\E_p[ |q-p|^j/p^{j}  \mathbf{1}_{\{E_4\}}  ] \\
	\leq& D_{|\chi|^j}(p\|r)  +  D_{|\chi|^j}(p\|q) +   D_{|\chi|^j}(q\|r) + D_{|\chi|^j}(q\|p) \leq 4\xi(j).
	\end{align*}
\end{proof}

\begin{lemma}[Lemma~\ref{lem:triangular_of_diff_pow} Restated]
	Let $X,Y$ be nonnegative random variables,  for any $j\geq 1$
	$$ \E[ |X - Y|^{j}] \leq \E [X^j] + \E[Y^j].$$
\end{lemma}
\begin{proof}
	Using that the $X,Y\geq 0$
	\begin{align*}
	\E[ |X - Y|^{j}]  &=  \E[ (X - Y)^{j} \mathbf{1}(X\geq Y) ] +   \E[ (X - Y)^{j} \mathbf{1}(X < Y) ]   \\
	&\leq \E[ X^j \cdot \mathbf{1}(X\geq Y) ]  + \E\left[ Y^j \cdot \mathbf{1}(X<Y)  \right]  \leq \E[ X^j ] +  \E[ Y^j] 
	\end{align*}
\end{proof}

\begin{lemma}[Lemma~\ref{lem:pure_dp_bound} Restated]
	Let $X,Y$ be nonnegative random variables and with probability $1$, $e^{-\varepsilon} Y \leq X \leq e^{\varepsilon} Y $. Then for any $j\geq 1$
	$$\E[ |X - Y|^{j}] \leq \E[Y^j](e^{\varepsilon} - 1)^j$$
\end{lemma}
\begin{proof}
	The multiplicative bound implies that:  $- Y (1-e^{-\varepsilon})\leq X-Y \leq Y (e^{\varepsilon}-1)$, which gives that with probability $1$
	$$|X-Y| \leq \max\{ e^{\varepsilon}-1,  1- e^{-\varepsilon}\}Y=   (e^{\varepsilon}-1)Y,$$
	and the claimed result follows. 
\end{proof}

\subsection{Proof of Corollary~\ref{cor:ext}}
\begin{corollary} [Corollary~\ref{cor:ext} Restated]
	Let $\lfloor \cdot \rfloor$ and $\lceil \cdot \rceil$ denotes the floor and ceiling operators
	$${K_{\cM}( \lambda)} \leq (1-\lambda + \lfloor \lambda \rfloor) K_{\cM}(\lfloor \lambda \rfloor)  + (\lambda -  \lfloor \lambda \rfloor) K_{\cM}(\lceil \lambda \rceil).$$
\end{corollary}
\begin{proof}
	The result is a simple corollary of the convexity of the CGF. Specifically, take $\lambda_1 =  \lfloor \lambda\rfloor$, $\lambda_2 = \lceil \lambda\rceil$ and $v :=  \lambda -  \lfloor \lambda \rfloor$. Note that $\lambda = (1-v)\lfloor \lambda\rfloor + v \lceil \lambda\rceil$. The result follows from the definition of convexity.
\end{proof}


\subsection{Improving the Bound in Theorem~\ref{thm:main}} \label{app:tight}
We note that we can improve the bound in Theorem~\ref{thm:main} under some additional assumptions on the RDP guarantee. We formalize this idea in this section. We use $d(X,X') \leq 1$ to represent neighboring datasets. We start with some additional conditions on the mechanism $\cM$ as defined below.
\begin{definition}[Tightness and Self-consistency]\label{def:assumption}
	We say a mechanism $\cM$ and its corresponding RDP privacy guarantee $\epsilon_\cM(\cdot)$ are  \emph{tight} if 
	$\max_{X,X':  d(X,X')\leq 1 } D_{\ell}( \cM(X)\|\cM(X')) = \epsilon_\cM(\ell)$ for every $\ell= 1,2,3,...$
	We say that a tight pair $(\cM,\epsilon_\cM(\cdot))$ is \emph{self-consistent} with respect to $|\chi|^\alpha$-divergence, if 
	\begin{align*}
	&\Big(\cap_{\ell=1,2,...,\alpha} \argmax_{X,X':  d(X,X')\leq 1}D_{\ell}( \cM(X)\|\cM(X')) \Big) \cap \argmax_{X,X':  d(X,X')\leq 1} D_{|\chi|^\alpha}( \cM(X)\|\cM(X'))  \neq  \emptyset.
	\end{align*}
\end{definition}

The tightness condition requires that the RDP function  $\epsilon_{\cM}(\cdot)$ to be attainable by two distributions induced by a pair of adjacent datasets and the self-consistency condition requires that \emph{the same} pair of distributions attains the maximal $|\chi|^\alpha$-divergence for a given range of parameters. Self-consistency is a non-trivial condition in general but it is true in most popular cases such as the Gaussian mechanism, Laplace mechanism, etc., where we know the R\'enyi divergence analytically and the difference of two datasets are characterized by one numerical number, e.g., sensitivity. (See Appendix~\ref{app:selfconsistency} for a discussion.) 

Define, 
$$B(\epsilon,l) := \Delta^{(l)}\left[e^{(\cdot-1)\epsilon(\cdot)}\right](0) = \sum_{i=0}^l (-1)^{i} \binom{l}{i} e^{(i-1)\epsilon(i)},$$ 
as the $l$th order forward finite difference (see~\eqref{eqn:fwd}) of the functional $e^{(\cdot-1)\epsilon(\cdot)}$ evaluated at $0$.

\begin{theorem}[Tighter RDP Parameter Bounds] \label{thm:tight}
	Given a dataset of $n$ points drawn from a domain $\cX$ and a (randomized) mechanism $\cM$ that takes an input from $\cX^{m}$ for $m \leq n$, let the randomized algorithm $\cM\circ \subsample$ be defined as: (1) $\subsample$: subsample without replacement $m$ datapoints of the dataset (sampling parameter $\gamma = m/n$), and (2) apply $\cM$: a randomized algorithm taking the subsampled dataset as the input.  If $\cM$ obeys $(\alpha,\epsilon(\alpha))$-RDP and additionally the RDP guarantee is tight and $(\alpha+1)$-self-consistent as per Definition~\ref{def:assumption},  then for all integer $\alpha \geq 2$, this new randomized algorithm $\cM\circ \subsample$ obeys $(\alpha,\epsilon'(\alpha))$-RDP where, 
	\begin{multline*}
	\epsilon'(\alpha)  \leq \frac{1}{\alpha-1}\log\bigg( 1 +  \gamma^2{\alpha \choose 2} \min\Big\{ 4(e^{\epsilon(2)}-1),
	e^{\epsilon(2)} \min\{2, (e^{\epsilon(\infty)}-1)^{2} \} \Big\}  \\ +  4\sum_{j=3}^{\alpha} \gamma^j {\alpha \choose j} \sqrt{B(\epsilon,2\lfloor j/2\rfloor)) \cdot  B(\epsilon,2\lceil j/2\rceil)} \bigg).
	\end{multline*}
\end{theorem}

\noindent\textbf{Proof Idea.} The proof is identical to that of Theorem~\ref{thm:main} as laid out in Appendix~\ref{app:mainproof}. The part where it differs is in Part 3, i.e., bounding $\zeta(j)^j$ using RDP. As a result of the assumptions in Definition~\ref{def:assumption}, we know that there exist a pair of adjacent data sets, which give rise to a pair of distribution $p$ and $q$, that simultaneously achieves the upper bound in the definition of both $\xi(j)$ and $\epsilon(j)$ divergences for all $j$ of interest.
For even $j$, the $\chi^j$-divergence can be written in an analytical form as a R\'enyi divergence \citep{nielsen2014chi} using a binomial expansion. Using Lemma~\ref{lem:ternary2binary} along with this expansion, gives rise to the $4\Delta^{(j)}[ e^{(\cdot -1) \epsilon(\cdot)}](0) = 4B(\epsilon,j)$ bound for even $j$.  For odd $j$, we reduce it to the even $j$ case through the Cauchy-Schwartz inequality 
\begin{align*}
\E_q[|p/q-1|^j] &=  \E_q[|p/q-1|^{(j-1)/2}|p/q-1|^{(j+1)/2}] \leq \sqrt{\E_q[(p/q-1)^{j-1}] \E_q[(p/q-1)^{j+1}]},
\end{align*}
where each of the term in the square root can now be bounded by the binomial expansion. Putting these together, one notices that one can replace  $e^{(j-1)\epsilon(j)}\min\{2,(e^{\epsilon(\infty)}-1)^{j}\}$ with a more exact evaluation given by $4\sqrt{B(\epsilon,2\lfloor j/2\rfloor)) \cdot  B(\epsilon,2\lceil j/2\rceil)}$ in the bound of Theorem~\ref{thm:main}. We use this bound only for $j \geq 3$ because for $j=2$, as discussed in Appendix~\ref{app:mainproof}, we have an alternative way of bounding $\zeta(2)$ that does not require these additional assumptions. 

\section{Asymptotic Approximation of R\'enyi Divergence for Subsampled Gaussian Mechanism} \label{app:gaussasym}
In this section, we present an asymptotic upper bound on the R\'enyi divergence for the subsampled Gaussian mechanism. The results from this section are also used in our numerical experiments detailed in Section~\ref{sec:exp}.

Let $\cX$ denote the input domain. Let $f : \mathcal{X} \rightarrow \Theta$ be some statistical query. We consider a subsampled Gaussian mechanism which releases the answers to $f$ by adding Gaussian noise to the mean of a subsampled dataset. In this case, the output $\theta$ of the subsampled Gaussian mechanism is a sample from $\cN(\mu_J,\sigma^2/|J|^2)$ where $\mu_J$ is short for $\mu(X_J) := \frac{1}{|J|} \sum_{i\in J} f(x_i)$ and $J$ is a random subset of size $\gamma n$. The distribution of $J$ induces a discrete prior distribution of $\mu_J$. Without loss of generality, we assume that $f(x_i)\leq 1/2$, which implies that the global sensitivity of  $\mu$ is $1/|J|$. By the sampling without replacement version of the central limit theorem\footnote{Under boundedness of $f(x_i)$, the regularity conditions holds.}, 
$\sqrt{|J|} (\mu(X_J)- \frac{1}{n}\sum_{i=1}^n f(x_i))$  converges in distribution to $\cN( 0, \frac{1}{n}\sum_{i=1}^n (f(x_i) - \mu(X))^2 )$. In other words, the distribution of $\theta$ asymptotically converges to
$$
\cN \left (\frac{1}{n}\sum_{i=1}^n f(x_i) , \frac{1}{n|J|}\sum_{i=1}^n (f(x_i) - \mu(X))^2+  \frac{\sigma^2}{|J|^2} \right ).
$$
This allows us to use the analytical formula of the R\'enyi divergence between two Gaussians  (see Appendix~\ref{app:expfamily_renyi}) as an asymptotic approximation of the R\'enyi divergence between the more complex mixture distributions. We disclaim that this is a truly asymptotic approximation and should only be true when $|J|,n\rightarrow \infty$ and $\gamma = |J|/n \rightarrow 0$, but it is nevertheless interesting as it allows us to understand the dependence of different parameters in the bound. One important observation is that the part of the variance due to the dataset can be either bigger or smaller than that of the added noise, and this could imply a vastly different R\'enyi divergence. We give examples here of two contrasting situations.
\begin{example}[Gaussian approximation - a ``bad'' data case]\label{exp:asymp_approx_worst}
	Let $f(x_1)=f(x_2)=...=f(x_{n-1}) = f(x_n)=-1/2$ for the elements in $X'$, and for $X$ the only difference (from $X'$) is that in $X$ we have $f(x_n)=1/2$. Then the two asymptotic distributions are  $p=\cN(-\frac{1}{2} + \frac{1}{n}, \frac{n-1}{n^2 |J|} + \frac{\sigma^2}{|J|^2})$ and $q = \cN(-\frac{1}{2},\frac{\sigma^2}{|J|^2})$, and the corresponding R\'enyi divergence equals
	$$
	D_\alpha(p\|q)  =  \begin{cases}
	+\infty  &\text{ if }\alpha \geq \frac{\sigma^2}{\gamma}\frac{n}{n-1} + 1,\\
	\frac{\alpha \gamma^2}{2\sigma^2} \left (\frac{\alpha^*}{\alpha^*-\alpha} \right )+\frac{1}{2}\log\left( \frac{\alpha^*-1}{\alpha^*} \right)+  \frac{1}{2(\alpha-1)}\log(\frac{\alpha^*}{\alpha^*-\alpha}) & \text{ otherwise}.
	\end{cases}
	$$
\end{example}

\begin{example}[Gaussian approximation - a ``good'' data case]
	Let $n$ be an odd number, and let $X'$ be such that $f(x_i)=1/2$ for $i\leq \lfloor n/2 \rfloor$ and $f(x_i)=-1/2$ otherwise, and for $X$ the only difference (from $X'$) is that in $X$ we have $f(x_n)=1/2$. The two asymptotic distributions are  $p=\cN(\frac{1}{2n}, \frac{\sigma^2}{|J|^2} + \frac{1}{4|J|} - \frac{1}{4n^2|J|})$ and $q = \cN(-\frac{1}{2n}, \frac{\sigma^2}{|J|^2} + \frac{1}{4|J|} - \frac{1}{4n^2|J|})$, and the corresponding R\'enyi divergence equals
	$$D_\alpha(p\|q)  =  \frac{\alpha\gamma^2}{2\sigma^2 + \gamma (n-n^{-1})/2}.$$ 
\end{example}
The first example (a ``bad'' data case) is closely related to our construction in the proof of Proposition~\ref{prop:lowerbound}. For $\alpha \ll \sigma^2/\gamma$, the example shows an $O(\alpha\gamma^2/\sigma^2)$  rate, matching our upper bound from Theorem~\ref{thm:main} (see Remark ``Bound under Additional Assumptions'' in Section~\ref{sec:amp}) in the small $\alpha$, large $\sigma$ regime.  
The second example corresponds to a ``good'' data case where the dataset has a variety of different datapoints, and as we can see, the variance of the asymptotic distribution that comes from subsampling the dataset dominates the noise from Gaussian mechanism and the per-instance RDP loss for this particular pair of $X$ and $X'$ can be $\gamma n$ times smaller than the bad case.


\section{Discrete Difference Operators and Newton's Series Expansion}\label{app:discrete_calculus}
In this section, we provide more details of the discrete calculus objects that we used in the proof, and also illustrate how the interesting identity \eqref{eq:discrete_mgf} comes about.

\noindent\textbf{Discrete Difference Operators.}
Discrete difference operators are linear operators that transform a function into its discrete derivatives.
Let $f$ be a function $\R\rightarrow \R$, the first order forward difference operator of $f$ is a function such that 
$$\Delta[f] (x) =  f(x+1) - f(x).$$

The $\alpha$th order  forward difference operator $\Delta^{(\alpha)}$ can be constructed recursively by
$$
\Delta^{(\alpha)} = \Delta \circ \Delta^{(\alpha-1)} 
$$
for all $\alpha=1,2,3,...$ with $\Delta^{(1)}:= \text{Id}$.

The forward difference operators are linear transformation of functions that can be thought of as a convolution (denoted by $\star$) with a linear combination of Dirac-delta functions ($\delta_{\rm dirac}$), which we call filters.
$$
\Delta[f]   =    f  \star  (\delta_{\rm dirac}(x-1)-\delta_{\rm dirac}(x)).
$$
From the linear combination point of view, the first order forward difference operator is the linear combination of the (infinite) basis functions of Dirac-delta functions supported on all integers with coefficient sequence $ [...,0,-1,1,0,...]$. This sequence of coefficients uniquely defines the difference operators.
For example, when $\alpha=2$, the coefficients that construct operator $\Delta^{(\alpha)} $ are
$$\dots,0,0,1,-2,1,0,0\dots$$
and when $\alpha=3$ and $\alpha=4$, we get
$$
\dots,0,0,-1,3,-3,1,0,0\dots
$$
and 
$$
\dots,0,0,1,-4,6,-4,1,0,0\dots
$$
respectively.
In general,  these convolution operators can be constructed by Pascal's triangle of the $\alpha$th order, or simply the binomial coefficients with alternating signs.

When computing the bound in Theorem~\ref{thm:main} we need to calculate $\Delta^{(\ell)}[f](0)$ for all integer $\ell \leq \alpha$. The recursive definition of the bound above allows us to compute all finite differences up to order $\alpha$ by $O(\alpha^2)$ evaluation of $f$ rather than the na\"ive direct calculation of $O(\alpha^3)$. In Appendix~\ref{app:ana_moment_accountant} we will describe further speed-ups with approximate evaluation.

\noindent\textbf{Newton Series Expansion.}
Newton series expansion is the discrete analogue of the continuous Taylor series expansion, with all derivatives replaced with discrete difference operators and all monomials replaced with falling factorials.

Consider infinitely differentiable function $f: \R\rightarrow \R$. The Taylor series expansion of $f$ at $0$ and the Newton series expansion of $f$ at $0$ are respectively:
\begin{align*}
f(x)  &=  f(0) +  \frac{\partial}{\partial x}[f](0)  x  +   \frac{\partial^2}{\partial x^2}[f](0) \frac{ x^2}{2!}   + ... + \frac{\partial^k}{\partial x^k}[f](0) \frac{ x^k}{k!}  + ...\\
f(x)& = f(0) + \Delta^{(1)}[f](0) x  + \Delta^{(2)}[f](0)  \frac{ x(x-1)}{2!} +... + \Delta^{(k)}[f](0) \frac{(x)_k}{k!} + ...
\end{align*}
where $(x)_k$ denotes the falling factorials $x(x-1)(x-2)...(x-k+1)$. For integer $x$, it is clear that the Newton's series expansion has a finite number of terms.

\section{On Tightness and Self-consistency Guarantees}\label{app:selfconsistency}
When specifying a sequence of RDP guarantees for $\cM$ in terms of $\sup_{X,X': d(X,X') \leq 1} D_{\alpha}(\cM(X)\|\cM(X')) \leq \epsilon(\alpha)$ it really matters whether $\epsilon(\alpha)$ is the exact analytical form of some underlying pairs of distributions induced by a pair of adjacent datasets $X,X'$ or just a sequence of conservative estimates. If it is the latter, then it is unclear at which $\alpha$ the slacks are bigger and at which $\alpha$ the slacks are smaller. And the sequence of $\epsilon(\cdot)$ might not be realizable by any pairs distributions. 
For example, if we use a polynomial upper bound of $\epsilon(\cdot)$, we know from the theory of CGF that no distribution have a CGF of polynomial order higher than $2$ and the only distribution that has polynomial order exactly two is the Gaussian distribution~\citep{lukacs1970characteristic}.

In this section, we provide an example proof that the analytical R\'enyi DP bound of the Gaussian mechanisms (defined in Section~\ref{sec:background}) are self-consistent. Again for simplicity, for the Gaussian mechanism, we assume that the sensitivity of function $f$ is $1$.
\begin{lemma} \label{lem:tight}
	For the Gaussian mechanism, $\epsilon(\alpha) = \alpha/(2\sigma^2)$ is tight and self-consistent.
\end{lemma}
\begin{proof}
	The Gaussian mechanism with variance $\sigma^2$ has a tight RDP parameter bound $\epsilon(\alpha) =  \frac{\alpha}{2\sigma^2}$~\citep{gil2013renyi}. This is achieved by the distributions $\cN(0,\sigma^2)$ and $\cN(1,\sigma^2)$. 
	
	For self-consistency, it suffices to show that the $|\chi|^\alpha$-divergence's maximum for every even $\alpha$ are also achieved by the same pair of distributions. Consider $q=\cN(0,\sigma^2)$  and $p=\cN(\mu,\sigma^2)$  for $0 \leq \mu\leq 1$
	$$D_{|\chi|^\alpha}(p\|q ) = \E_q[ (p/q-1)^\alpha]  =  \E_q[ ( e^{-\frac{-2x\mu + \mu^2}{2\sigma^2}} -1  )^{\alpha}] = \Delta^{(\alpha)}[ e^{(\ell^2-\ell) \mu^2}](0)$$
	Take derivative w.r.t. $\mu$, we get 
	$$
	2\mu(\ell^2-\ell)	\Delta^{(\alpha)}[ \E_q[e^{(\ell^2-\ell) \mu^2}]](0)  \geq 0
	$$
	for $\mu>0$.
	In other words, the divergence is monotonically increasing in $\mu$.
\end{proof}

In general, verifying the self-consistency is not straightforward, but since $|\chi|^\alpha$-divergence is a proper $f$-divergence, it is jointly convex in its arguments. When the set of distributions is a convex polytope, it suffices to check for this condition at all the vertices of the polytope. 

\section{Other Properties of Ternary-$|\chi|^\alpha$-DP}\label{app:properties_of_chi_DP}
When $\alpha = 1$, both the binary- and ternary-$|\chi|^\alpha$-divergence reduces to the total variation distance. When $\alpha=2$ the binary-$|\chi|^\alpha$-divergence become the $\chi^2$-distance.

The following lemma shows that we can convert binary-$|\chi|^\alpha$-DP (and therefore, ternary-$|\chi|^\alpha$-DP)  to  the more standard $(\epsilon,\delta)$-DP using the tail bound of a privacy random variable. 
\begin{lemma}[$|\chi|^\alpha$-differential privacy $\Rightarrow$ $(\epsilon,\delta)$-DP]
	If an algorithm is $\xi$-binary-$|\chi|^\alpha$-DP, then it is also 
	$
	\left(\epsilon, \Big(\frac{\xi(\alpha)}{e^\epsilon-1}\Big)^\alpha \right)
	$-DP for all $\epsilon>0$ and equivalently, $(\log \xi(\alpha) -1 + \frac{\log (1/\delta)}{\alpha},\delta)$ for all $\delta>0$.
\end{lemma}
\begin{proof}
	By Markov's inequality,
	$$
	\Pr[ |p/q -1| > t ] \leq \E[ |p/q -1|^\alpha ]/ t^\alpha  =  \left(\frac{\xi(\alpha)}{ t}\right)^\alpha.
	$$
	The results follows from changing the variable from $p/q$ to $e^{\log(p/q)}$. 
\end{proof}

The following lemma shows that we can bound the above by a quantity that depends on the R\'enyi divergence and the Pearson-Vajda divergence. It also generalizes Lemma~\ref{lem:pure_dp_bound} that we used in the proof of Theorem~\ref{thm:main}.
\begin{lemma}
	Let $p,q,r$ are three distributions.
	For all conjugate pair $u,v\geq 1$ such that $1/u + 1/v =1$, and all integer $j\geq 2$ we have that 
	$$
	\E_r \left [ \left (\frac{|p-q|}{r} \right )^j \right]  \leq e^{(j-1)D_{(j-1)v+1}(q\|r)}  D_{|\chi|^{ju}}(p \| q)^{1/u}.   
	$$  
\end{lemma}
\begin{proof}
	The proof is a straightforward application of the H\"older's inequality.
	\begin{align*}
	\E_r\left[ \Big(\frac{|p-q|}{r}\Big)^j \right]  =&  \int  r  \left (\frac{q}{r} \right )^j   \left |\frac{p}{q} - 1  \right|^j d\theta 
	\explain{=}{\text{Change of measure}} \int  q \left (\frac{q}{r} \right)^{j-1}  \left |\frac{p}{q} - 1 \right |^j d\theta \\
	\explain{\leq}{\text{H\"older}}&  \left(\E_q \left [ \left (\frac{q}{r} \right)^{(j-1)v} \right ]\right)^{1/v}  \left(\E_q \left [ \left (\frac{p}{q}-1\right )^{ju} \right ]\right)^{1/u}\\ 
	=&  e^{(j-1) D_{(j-1)v+1}(q\|r)}   D_{|\chi|^{ju}} (p\|q)^{1/u}.
	\end{align*}
\end{proof}
\begin{remark} 
	When we take $v=\infty$ and $u=1$, we recover the result from Lemma~\ref{lem:pure_dp_bound}. 
	When we take $u=v=2$, this guarantees that $ju$ is an even number and the above results becomes 
	$$
	\E_r \left [ \left (\frac{|p-q|}{r} \right )^j \right ] \leq e^{(j-1)D_{2j-1}(q\|r)}\sqrt{ \Delta^{(2j)}[e^{(\cdot-1)D_{(\cdot)}(p\|q)}](0) },
	$$
	where $ \Delta^{(2j)}$ is the finite difference operator of order $2j$.  Note that $e^{(\cdot-1)D_{(\cdot)}(q\|r)}$ can be viewed as the moment generating function of the random variable $\log(p(\theta)/q(\theta))$ induced by $\theta\sim q$. The $2j$th order discrete derivative of the MGF at $0$ is $\E_q[ (\frac{p}{q}-1)^{2j}]$, which very nicely mirrors the corresponding $2j$th order continuous derivative of the MGF evaluated at $0$, which by the property of an MGF is $\E_q[\log(p/q)^{2j} ]$.
\end{remark}

\section{Analytical Moments Accountant and Numerically Stable Computation} \label{app:ana_moment_accountant}
In this section, we provide more details on the \emph{analytical moments accountant} that we described briefly in Section~\ref{sec:ana_moment_acct}.
Recall that the analytical moments accountant is a data structure that one can attach to a dataset to keep track of the privacy loss over a sequence of differentially private data accesses. The data structure caches the CGF of the privacy random variables in symbolic form and permits efficient $(\epsilon,\delta)$-DP calculations for any desired $\delta$ or $\epsilon$. Here is how it works.

Let $\cM_1,\cM_2,..,\cM_k$ be a sequence of (possibly adaptively chosen) randomized mechanisms that one applies to the dataset and the $K_{\cM_1},...,K_{\cM_k}$ be the corresponding CGF. The analytical moments accountant maintains $K = K_{\cM_1}+...+K_{\cM_k}$ in symbolic forms and it can evaluate $K(\lambda)$ at any $\lambda > 0$.   The two main usage of the analytical moments accountant are for keeping track of: (a) RDP parameter $\epsilon(\alpha)$ for all $\alpha$, and (b) $(\epsilon(\delta),\delta)$-DP for all $0 \leq \delta<1$, for a heterogeneous sequence of adaptively chosen randomized mechanisms. The conversion to RDP is straightforward using the one-to-one relationship between CGF and RDP (see Remark~\ref{rmk:cgf2renyi}) with the exception of RDP at $\alpha = 1$ (Kullback Leibler-privacy) and $\alpha = +\infty$ (pure DP), which we keep track of separately. The conversion to  $(\epsilon,\delta)$-DP is obtained by solving the univariate optimization problems described in~\eqref{eq:eps_from_delta} and~\eqref{eq:delta_from_eps}. 

We note that our analytical moments accountant is conceptually the same as the moments accountant used by \citet{abadi2016deep} and the RDP composition used by \citet{mironov2017renyi}. Both prior work however considered only a predefined discrete  list of orders $\lambda$ (or $\alpha$'s). Our main difference is that, for every mechanism, we keep track of the CGF for all $\lambda \in \R_+$ at the same time. 

In the remainder of the section, we will describe specific designs of this data structure and substantiate our claims described earlier in Section~\ref{sec:ana_moment_acct}.

\noindent\textbf{Space and Time Complexity for Tracking Mechanisms and for $(\epsilon,\delta)$-DP Query.}
We start by analyzing the space and time complexity of basic operations of this data structure.
\begin{proposition}
	The analytical moments accountant takes $O(1)$ time to compose a new mechanism. 	At any point in time after the analytical moments accountant has been declared and in operation, let the total number of unique mechanisms that it has seen so far be $L$. Then the analytical moments accountant takes $O(L)$ space . The CGF queries (at a given $\lambda$) takes time $O(L)$.  $(\epsilon,\delta)$-DP query to accuracy $\tau$ (in terms of absolute difference in the argument $|\lambda-\lambda^*|$) takes time $O(L)$ and $O(L\log(\lambda^*)/\tau)$ CGF evaluation calls respectively, where $\lambda^\ast$ is the corresponding minimizer in \eqref{eq:eps_from_delta} or \eqref{eq:delta_from_eps}. 
\end{proposition}
\begin{proof}
	We keep track of a dictionary of $\lambda$ functions where the (key,value)-pair is effectively ($\cM, (K_{\cM}, c_{\cM})$) where $K_{\cM}$ is a function that returns the CGF given any positive input, and $c_{\cM}$ is the coefficient denoting how many times $\cM$ appeared. This naturally allows $O(1)$ time to add a new mechanism and $O(L)$ space.
	
	Since CGFs composes by simply adding up the functions, the overall CGF is $\sum_{i=1}^L  c_{\cM_i} K_{\cM_i}$. Evaluating this function takes $L$ CGF queries. We think of the problems of solving for $\epsilon$ given $\delta$ and solving for $\delta$ given $\epsilon$ as zeroth order optimization problem using these queries.  These problems are efficiently solvable due to the geometric properties of CGFs that we mention in Section~\ref{sec:background} and Appendix~\ref{app:CGF_properties}.
	
	When solving for $\epsilon$ given $\delta$, we keep doubling the candidate $\lambda_{\max }$ and calculating  $\frac{1/\delta+ K_{\cM}(\lambda_{\max })}{(\lambda_{\max })} - \frac{1/\delta+ K_{\cM}(\lambda_{\max }-1)}{(\lambda_{\max }-1)}$ until we find that it is positive. This procedure is guaranteed to detect a bounded interval that guarantees to contain $\lambda^*$ in $O(\log \lambda^*)$ time thanks to the monotonicity of RDP.  Then we do bisection to find the optimal $\lambda^\ast$, using the unimodal property of the objective function. Note that $\lambda_{\max} \leq 2\lambda^\ast$. This ensures that the oracle evaluation complexity to find a $\tau$-optimal solution (i.e., to within accuracy $\tau$) of $\lambda^\ast$ is $O(\log(\lambda^\ast/\tau)$. 
	We can solve for $\delta$ given $\epsilon$ using the same bisection algorithm with the same time complexity, by using the fact that \eqref{eq:delta_from_eps} is a log-convex problem.
\end{proof}
The results are compared to a na\"ive implementation of the standard moments accountant that keeps track of an array of size $\lambda_{\max}$ and handles $\delta\Rightarrow\epsilon$ queries without regarding the geometry of CGFs.  The latter will take $O(\lambda_{\max })$ time and space for tracking a new mechanism, and $O(\lambda_{\max})$ time to find an $1$-suboptimal solution. In addition, it does not allow a dynamic choice of $\lambda_{\max }$. The analytical moments accountant described here, despite its simplicity, is an exponential improvement over the na\"ive version, besides being more flexible and adaptive.

There are still several potential problems. First, the input could be an upper bound which may not be an actual CGF function of any random variable, therefore breaking the computational properties.
Secondly, when we need to handle subsampled mechanisms, even just evaluating the RDP bound in Theorem~\ref{thm:main} for once at $\alpha$ will cost $O(\alpha^2)$ (therefore $O(\lambda^2)$). Lastly, the quantities in the bound of Theorem~\ref{thm:main} could be exponentially large and dealing them na\"ively will cause floating point numbers overflow or underflow. We address these problems below.

\noindent\textbf{``Projecting'' a CGF Upper Bound into a Feasible Set.} 
Note that an upper bound of the CGF does not necessarily have the standard properties associated with CGF that we note in Appendix~\ref{app:CGF_properties}, however, we can ``project'' it to another valid upper bound using the proposition below so that it satisfies the properties from Appendix~\ref{app:CGF_properties}.
\begin{proposition}
	Let $\bar{K}_{\cM}$ be an upper bound of $K_{\cM}$, there is a functional $F$ such that $F[\bar{K}_{\cM}] \leq K_{\cM}$ and $F[\bar{K}_{\cM}]$ obeys that $F[\bar{K}_{\cM}]$ is convex, monotonically increasing, evaluates to $0$ at $0$, and
	$
	\frac{1}{\lambda}F[\bar{K}_{\cM}](\lambda)
	$
	is monotonically increasing on $\lambda \geq 0$. 
\end{proposition}
\begin{proof}
	We prove by constructing such an $F$ explicitly. First define $g := \mathrm{convexhull}( \bar{K}_{\cM})$. By definition, $g$ is the pointwise largest convex function that satisfies the given upper bound.
	Secondly, we find the largest $\beta$ such that $\beta\lambda \leq g(\lambda), \,\, \forall \lambda$. Let the smallest $\lambda$ such that  $g(\lambda) = \beta\lambda$ be $\tilde{\lambda}$. Then, we define 
	$$
	F[\bar{K}_{\cM}](\lambda)   =  \begin{cases}
	0 & \text{when } \lambda\leq 0,\\
	\beta \lambda &  \text{when } 0<\lambda\leq \tilde{\lambda},\\
	g(\lambda) & \text{when }  \lambda > \tilde{\lambda}.
	\end{cases}
	$$
	Clearly, this is the largest function that satisfy the shape constraints, and therefore must be an upper bound of the actual true CGF of interest.
\end{proof}
This ensures that if we replace $K_{\cM}$ with $F[\bar{K}_{\cM}] $  for any upper bound $\bar{K}_{\cM}$, the computational properties of~$\eqref{eq:eps_from_delta}$ and~$\eqref{eq:delta_from_eps}$ remain unchanged.

\noindent\textbf{Approximate Computation of Theorem~\ref{thm:main}.}
The evaluation of the RDP itself for a subsampled mechanism according to our bounds in Theorem~\ref{thm:main} could still depend polynomially in $\alpha$. We resolve this by only calculating the bound exactly up to a reasonable $\alpha_{thresh}$ and then for $\alpha>\alpha_{thresh}$, we use an optimization based-upper bound.

Noting that the expression in Theorem~\ref{thm:main} can be written as a log-sum-exp or $\mathrm{softmax}$ function of $\alpha+1$ items, where the $j$th item corresponds to:
$$
\log {\alpha \choose  j} + j \log  \gamma + j \log \zeta(j).
$$
Here, $\zeta(j)$ is the smallest of the upper bounds that we have of the ternary $|\cX|^j$-privacy of order $j$ using RDP.

For any vector $x$ of length $\alpha+1$ we can use the following approximation:
$$
\max(x) \leq \mathrm{softmax}( x)  \leq \max(x) + \log(\alpha).
$$
When $\exp(x-\max(x))$ is dominated by a geometric series (which it often is for most mechanism $\cM$ of interest), then we can further improve $\log(\alpha)$ by something independent to $\alpha$.

The $\max(x)$ can be solved efficiently in $O(\log(\alpha))$ time as the function can have at most two local minima. This observation follows from the fact that $\log \zeta(j)$ (or any reasonable upper bound of it) is monotonically increasing,  $j \log  \gamma$ is monotonically decreasing, and that $\log {\alpha \choose j}$ is unimodal. 
Furthermore, we use the Stirling approximation for $ \log{\alpha \choose j}$ when $\alpha$ is large.

\noindent\textbf{Numerical Stability in Computing the bound in Theorem~\ref{thm:main}.}
Since log-sum-exp is involved, we use the standard numerically stable implementation of the log-sum-exp function via: $\log(\sum_i \exp(x_i))  = \max_{j}x_j  + \log(\sum_i\exp(x_i - \max_j(x_j)))$.

We also run into new challenges. For instance, the $\sum_{\ell=0}^j {j \choose \ell} (-1)^{j-\ell} e^{(\ell-1)\epsilon(\ell)}$ term involves taking structured differences of very large numbers that ends up being very small. We find that the alternative higher order finite difference operator representation $\Delta^{(j)} [e^{(\cdot-1)\epsilon(\cdot)}](0)$ and a polar representation of real numbers with a sign and $\log$ absolute value allows us to avoid floating point number overflow. However, the latter approach still suffers from the problem of error propagation and does not accurately compute the expression for large $j$.

To the best of our knowledge, the numerical considerations and implementation details of the moments accountant have not been fully investigated before, and accurately computing the closed form expression of $\chi^j$-divergences using R\'enyi Divergences for large $j$ remains an open problem of independent interest.

\section{Properties of Cumulant Generating Functions and R\'enyi Divergence}\label{app:CGF_properties}
In this section, we highlight some interesting properties of CGF, which in part enables our analytical moments accountant data structure described in Appendix~\ref{app:ana_moment_accountant}.
\begin{lemma}\label{lem:properties}
	The CGF of a random variable (if finite for $\lambda \in\R$), obeys that: 
	\begin{enumerate}
		\item[(a)] It is infinitely differentiable. 
		\item[(b)] $\frac{\partial}{\partial\lambda}K_\cM(\lambda)$ monotonically increases from the infimum to the supremum of the support of the random variable.
		\item[(c)] It is  convex (and strictly convex for all distributions that is not a single point mass).
		\item[(d)] $K_\cM(0)=0$, e.g., it passes through the origin.
		\item[(e)] The CGF of a privacy loss random variable further obeys that  $K_\cM(-1)=0$.  
	\end{enumerate}
\end{lemma}
These properties are used in establishing the computational properties of the analytical moments accountant as we have seen before.

We provide a first-principle proof of convexity (c), which is elementary and does not use a variational characterization of the R\'enyi divergence as in the Corollary 2 of  \cite{van2014renyi}.
\begin{proof}
	We use the definition of convex functions. 
	By definition, for all $\lambda \geq 0$, we have 
	\begin{align*}
	K_{\cM}( \lambda) = \log \E_{p}[e^{\lambda \log \frac{p(\theta)}{q(\theta)}}] = \log \E_p\left[\left(\frac{p(\theta)}{q(\theta)}\right)^{\lambda}\right].
	\end{align*}
	
	Let  $\lambda_1,\lambda_2\geq 0$ and $v\in[0,1]$.
	Take $\lambda = (1-v)\lambda_1  + v\lambda_2)/2$ and apply H\"older's inequality with the exponents being the conjugate pair $1/(1-v)$ and $1/v$: 
	\begin{align*}
	\E_p \left[\left(\frac{p(\theta)}{q(\theta)}\right)^{\lambda}\right] &
	=  \E_p \left[\left(\frac{p(\theta)}{q(\theta)}\right)^{(1-v)\lambda_1 + v \lambda_2}\right] =  \E_p \left[\left(\frac{p(\theta)}{q(\theta)}\right)^{(1-v)\lambda_1 }   \left(\frac{p(\theta)}{q(\theta)}\right)^{v \lambda_2}\right]\\
	&\leq \E_p\left[\left(\frac{p(\theta)}{q(\theta)}\right)^{\lambda_1  } \right]^{1-v}   \E_p\left[\left(\frac{p(\theta)}{q(\theta)}\right)^{\lambda_2 }\right]^{v}\\
	&=\exp[{K_{\cM}(\lambda_1 )}]^{1-v} \exp [{K_{\cM}(\lambda_2 )}]^{v}.
	\end{align*}
	Take logarithm on both sides, we get 
	$$
	K_{\cM}(  (1-v) \lambda_1 + v\lambda_2)  \leq  (1-v) K_{\cM}( \lambda_1) + v K_{\cM}( \lambda_2)
	$$
	and the proof is complete.
\end{proof}

\begin{corollary}\label{cor:convex}
	Optimization problem \eqref{eq:delta_from_eps} is log-convex. Optimization problem \eqref{eq:eps_from_delta} is unimodal / quasi-convex. 
\end{corollary}
\begin{proof}
	To see the first claim, check that the logarithm of \eqref{eq:delta_from_eps} is the sum of a convex function and an affine function, which is convex. To see the second claim, first observe $1/\lambda$ is monotonically decreasing in $\R_+$. It suffices to show that $\frac{K_{\cM}(\lambda)}{\lambda}$ (this is RDP! )  is monotonically increasing. 
	Let $\partial K_{\cM}(\lambda)$ be a subgradient of $K_{\cM}(\lambda)$, we can take the ``derivative'' of the function
	$$
	\lim_{\delta\rightarrow 0}\frac{1}{\delta} \left(\frac{K_{\cM}(\lambda+\delta)}{\lambda + \delta} - \frac{K_{\cM}(\lambda)}{\lambda} \right)  \geq \frac{ \partial K_{\cM}(\lambda)}{\lambda} -  \frac{K_{\cM}(\lambda)}{\lambda^2} \geq 0
	$$
	The last inequality follows from the first order condition of a convex function
	$$
	K_{\cM}(0)  \geq  K_{\cM}(\lambda)  + (0-  \lambda) \cdot \partial K_{\cM}(\lambda)
	$$
	and that $K_{\cM}(0) =0$.  
\end{proof}
The corollary implies that optimization problems defined in~\eqref{eq:eps_from_delta} and~\eqref{eq:delta_from_eps} have unique minimizers and they can be solved efficiently using bisection or convex optimization to arbitrary precision even if all we have is (possibly noisy) blackbox access to $K_{\cM}(\cdot)$ or its derivative.

\section{R\'enyi Divergence of Exponential Family Distributions and RDP}
\textbf{Exponential Family Distributions.}
Let $\theta$ be a random variable whose distribution parameterized by $\phi$. It is an exponential family distribution if the probability density function can be written as
$$
p(\theta;\phi)   =  h(\theta) \exp(\eta(\phi)^T T(\theta) -  F(\phi)).
$$
If we re-parameterize, we can rewrite the exponential family distribution as a \emph{natural} exponential family
$$
p(\theta;\eta)   =  h(\theta) \exp(\eta^T T(\theta) -  A(\eta))
$$
where the normalization constant $A$ is called the log-partition function.

\noindent\textbf{R\'enyi Divergence of Two Natural Exponential Family Distributions.}
Let $\cS$ be the natural parameter space, i.e., every $\eta\in\cS$ defines a valid distribution. Then for $\eta_1,\eta_2\in\cS$, the R\'enyi divergence between the two exponential family distribution $p_{\eta_1} := p(\theta;\eta_1)$ and $p_{\eta_2} := p(\theta;\eta_2)$  is:
\begin{enumerate}
	\item If $\alpha\notin \{ 0,1\}$ and $\alpha \eta_1 + (1-\alpha)\eta_2 \in \cS$, 
	$$
	D_\alpha(p_{\eta_1}\| p_{\eta_2}) = \frac{1}{\alpha-1}\log \left ( { \frac{ A(\alpha \eta_1 + (1-\alpha)\eta_2)}{ A(\eta_1)^\alpha A(\eta_2)^{1-\alpha}} } \right ).
	$$
	\item If $\alpha \notin\{0,1\}$ and $\alpha \eta_1 + (1-\alpha)\eta_1 \notin \cS$, 
	$$
	D_\alpha(p_{\eta_1}\| p_{\eta_2}) = +\infty
	$$
	\item If $\alpha = 1$,
	$$
	D_\alpha(p_{\eta_1}\| p_{\eta_2})  = D_{KL}(p_{\eta_1}\| p_{\eta_2}) =  (\eta_1-\eta_2)^T \nabla_\eta A(\eta_1) + A(\eta_2) - A(\eta_1),
	$$
	namely, the Kullback Liebler divergence of the two distributions and also the Bregman divergence with respect to convex function $A$.
	\item If $\alpha = 0$,
	$$
	D_\alpha(p_{\eta_1}\| p_{\eta_2})  = -\log( \Pr_{\eta_2} [p_{\eta_1}>0]).
	$$
\end{enumerate}

For example, the R\'enyi divergence between multivariate normal distributions $\cN(\mu_1,\Sigma_1), \cN(\mu_2,\Sigma_2)$ equals~\citep{gil2013renyi}
\begin{align*}
&D_\alpha(  \cN(\mu_1,\Sigma_1) \| \cN(\mu_2,\Sigma_2))\\
=& \begin{cases}
+\infty, \quad\quad\quad   \quad \quad\text{ if }  \Sigma_\alpha :=  \alpha\Sigma_2 + (1-\alpha)  \Sigma_1\text{ is not positive definite.}\\
\frac{\alpha}{2} (\mu_1-\mu_2)^T \Sigma_\alpha^{-1}  (\mu_1-\mu_2) - \frac{1}{2(\alpha-1)}\log \left (\frac{| \Sigma_\alpha |}{ |\Sigma_1|^{1-\alpha} |\Sigma_2|^{\alpha}} \right ), \text{ otherwise. }
\end{cases}
\end{align*}

\noindent\textbf{Exponential Family Mechanisms and its R\'enyi-DP.}\label{app:expfamily_renyi}
Let the differentially private mechanism to release $\theta$ be sampling from an exponential family. Let
$$
p(\theta)   =  h(\theta) \exp(\eta(X)^T T(\theta) -  A(\eta(X))) 
$$
denote the distribution induced by this differentially private mechanism on dataset $X$, and similarly let 
$$
q(\theta) = h(\theta) \exp(\eta(X')^T T(\theta) -  A(\eta(X'))).
$$
be the corresponding distribution when the dataset is $X'$.

In this case, the privacy random variable $\log (p/q)$ has a specific form
$$
\varphi(\theta)   =  [\eta(X)-\eta(X')]^TT(\theta) -  [A(\eta(X)) - A(\eta(X')) ].
$$
Using this, it can be shown that the $\alpha$-R\'enyi divergence between $p$ and $q$ is
\begin{align*}
D_\alpha(p\|q )  &=  \log \E_q\left[ e^{\alpha\varphi(\theta)} \right]^{\frac{1}{\alpha - 1} }\\
& =  \frac{1}{\alpha-1} \left[  A(\alpha\eta(X) + (1-\alpha)\eta(X') )  - \alpha A(\eta(X)) - (1-\alpha) A(\eta(X')) \right].
\end{align*}
A special case of the exponential family mechanisms of particular interest is the posterior sampling mechanisms where $\eta(X)$ has a specific form~\citep{geumlek2017renyi}.

To obtain RDP from the above closed-form R\'enyi divergence, it remains to maximize over two adjacent data sets $X,X'$.  We make a subset of the following three assumptions.
\begin{itemize}
	\item[(A)] Bounded parameter difference: $\sup_{X,X': d(X,X')\leq 1}\|\eta(X)-\eta(X')\|  \leq \Delta$ with respect a norm $\|\cdot\|$.
	\item[(B)] $(B,\kappa)$-Local Lipschitz: The log-partition function $A$ is $(B,\kappa)$-Local Lipschitz with respect to $\|\cdot\|$ if for all data set $X$ and all $\eta$  such that $\|\eta - \eta(X)\|\leq \kappa$, we have 
	$$
	A(\eta)  \leq  A(\eta(X)) + B\|\eta - \eta(X)\|.
	$$
	\item[(C)] $(L,\kappa)$-Local smoothness: The log-partition function $A$ is $(L,\kappa)$-smooth with respect to $\|\cdot\|$ if for all data set $X$ and all  $\eta$  such that $\|\eta - \eta(X)\|\leq \kappa$, we have
	$$
	A(\eta) \leq A(\eta(X))  + \langle \nabla A(\eta(X)),\eta - \eta(X)\rangle  + L\|\eta - \eta(X)\|^2.
	$$
\end{itemize}

The following proposition refines the results of \citep[Lemma~3]{geumlek2017renyi}.
\begin{proposition}[RDP of exponential family mechanisms]\label{prop:RDP_expfamily}
	Let $\cM$ is an exponential family mechanism that obeys Assumption (A)(B)(C) with parameter $\Delta,B,L,\kappa$ with a common norm $\|\cdot\|$. If in addition, $\kappa\geq \Delta$, then $\cM$ obeys $(\alpha,\epsilon(\alpha))$-RDP for all $\alpha\in (1, \kappa/\Delta+1]$ with
	$$\epsilon(\alpha)\leq \min\left\{ \frac{\alpha L\Delta^2}{2},  2B\Delta \right\}.$$
\end{proposition}
\begin{remark}
	We can view $B$ and $L$ as (nondecreasing) functions of $\kappa$.  For any fixed $\alpha$ of interest, we can optimize over all feasible choice of $\kappa$:
	$$
	\epsilon(\alpha)\leq \min_{\kappa:  \alpha\Delta \leq \kappa}\min\left\{ \alpha L(\kappa)\Delta^2,  2B(\kappa)\Delta \right\} =   \min\left\{  \alpha L(\alpha \Delta)\Delta^2, 2 B(\alpha\Delta)\Delta \right\}. 
	$$
	In fact, as can be seen clearly from the proof, $2 B(\alpha\Delta)\Delta$ can be improved to $[B((\alpha-1)\Delta) +  B(\Delta)]\Delta $.
\end{remark}
\begin{proof}[Proof of Proposition~\ref{prop:RDP_expfamily}]
	Assumption (A) implies that $\|\eta(X)-\eta(X')\|  \leq \Delta$.
	Note that for all $\alpha \leq \kappa/\Delta$, $\|\alpha\eta(X) + (1-\alpha)\eta(X') - \eta(X) \| \leq \kappa$.  Assumption (B) implies that 
	\begin{align*}
	A(\alpha\eta(X) + (1-\alpha)\eta(X') ) &\leq  A(\eta(X)) +  (\alpha-1) B \|\eta(X')-\eta(X))\| \leq  A(\eta(X) +  (\alpha-1) B \Delta,
	\end{align*}
	and that
	$$
	A(\eta(X'))  \leq A(\eta(X))  +  B \Delta.
	$$
	Substitute these into the definition of $D_\alpha(p\|q )$ we get that
	\begin{equation}\label{eq:rdp_expfamily_lipschitz}
	D_\alpha(p\|q )  \leq \frac{1}{\alpha-1}  [ A(\eta(X))+ (\alpha-1) B \Delta  - A(\eta(X))  +  (\alpha-1)B\Delta ] = 2B\Delta.
	\end{equation}
	
	Assumption (C) implies that for all $\alpha \leq \kappa/\Delta +1$
	\begin{align*}
	&A(\alpha\eta(X) + (1-\alpha)\eta(X') ) =  A(\eta(X) +  (\alpha-1) (\eta(X)-\eta(X'))) \\
	\leq& A(\eta(X))   +  (\alpha-1) \langle \nabla A(\eta(X), \eta(X)-\eta(X'))\rangle  +  \frac{(\alpha-1)^2 L}{2}\|\eta(X)-\eta(X')\|^2 \\
	\leq& A(\eta(X)) + (\alpha-1) \langle \nabla A(\eta(X), \eta(X)-\eta(X'))\rangle  +   \frac{(\alpha-1)^2 L\Delta^2}{2}
	\end{align*}
	where the last step uses Assumption (A).
	Assumption (C) also implies that
	\begin{align*}
	A(\eta(X')) - A(\eta(X) &\leq  \langle \nabla A(\eta(X), \eta(X')-\eta(X))\rangle + \frac{L \| \eta(X)-\eta(X')\|^2 }{2}  \\
	&\leq \langle \nabla A(\eta(X), \eta(X)-\eta(X'))\rangle + \frac{L \Delta^2 }{2}.
	\end{align*}
	Substitute these into the definition of $D_\alpha(p\|q )$ we get that
	\begin{align*}
	D_\alpha(p\|q )  \leq \frac{1}{\alpha-1}  \Big[& A(\eta(X))+ (\alpha-1) \langle \nabla A(\eta(X), \eta(X)-\eta(X'))\rangle  +   \frac{(\alpha-1)^2 L\Delta^2}{2}\\
	& - A(\eta(X)) +  (\alpha-1)  \langle \nabla A(\eta(X), \eta(X')-\eta(X))\rangle  +    \frac{(\alpha-1) L\Delta^2}{2}\Big] =  \frac{\alpha  L\Delta^2}{2},
	\end{align*}
	which, together with \eqref{eq:rdp_expfamily_lipschitz}, produces the bound as claimed.
\end{proof}
\end{document}